\documentclass{article}

 \usepackage[preprint]{neurips_2025}

\usepackage[utf8]{inputenc} % allow utf-8 input
\usepackage[T1]{fontenc}    % use 8-bit T1 fonts
\usepackage{hyperref}       % hyperlinks
\usepackage{url}            % simple URL typesetting
\usepackage{booktabs}       % professional-quality tables
\usepackage{amsfonts}       % blackboard math symbols
\usepackage{nicefrac}       % compact symbols for 1/2, e
\usepackage{graphicx}
\usepackage{microtype}      % microtypography
\usepackage{xcolor}         % colors

\usepackage{microtype}
\usepackage{graphicx}
\usepackage{subcaption}
\usepackage{booktabs} % for professional tables
\usepackage{hyperref}
\usepackage{algorithm}
\usepackage{algorithmic}
\usepackage{amsmath}
\usepackage{amssymb}
\usepackage{mathtools}
\usepackage{amsthm}
\usepackage{amsthm}
\usepackage{enumitem}
\usepackage{cancel}
\usepackage{subcaption} % in preamble
\usepackage{algorithm}

\usepackage[symbol]{footmisc}

\usepackage{enumitem}
\setlist{nosep} % Removes all vertical spacing
\theoremstyle{plain}
\newtheorem{theorem}{Theorem}
\newtheorem{proposition}[theorem]{Proposition}
\newtheorem{lemma}[theorem]{Lemma}

\theoremstyle{definition}

\theoremstyle{remark}

\renewcommand{\L}{\mathcal{L}}
\newcommand{\med}{\operatorname{med}}
\newcommand{\A}{\mathcal{A}}
\newcommand{\D}{\mathcal{D}}
\newcommand{\J}{\mathcal{J}}

\newcommand{\R}{\mathbb{R}}
\renewcommand{\d}{\mathrm{d}}

\newcommand{\xx}[2]{x^{(#1)}_{#2}}

\title{Neural Kolmogorov Equations: Parallelizable \\ Learning of Stochastic Dynamics under General Noise}
\author{%
  Arthur Bizzi \\
  EPFL, Imperial College\\ 
  \texttt{a.bizzi@imperial.ac.uk} \\
  \And
Olga Fink\\
EPFL \\
  \texttt{olga.fink@epfl.ch} \\
}

\begin{document}

\maketitle

\begin{abstract}
% Stochastic differential equations (SDEs) provide a fundamental framework for modelling dynamical systems subject to noise,
% uncertainty, with applications ranging from physics and biology to finance and climate science.
Neural stochastic differential equations (SDEs) have emerged as powerful tools for learning noisy or stochastic dynamics directly from data;
% , providing a framework for modelling systems subject to noise
however, existing approaches largely assume uncoupled and continuous noise, limiting their applicability to realistic stochastic drivers, and often scale poorly in time, requiring expensive autoregressive training. To address these limitations, we propose Neural Kolmogorov Equations (NKEs), a deterministic, infinite-dimensional reformulation of Neural SDEs based on the Kolmogorov Forward equation, transforming the learning problem from modelling individual stochastic trajectories to modelling the evolution of probability densities. NKEs learn general Lévy-type stochastic forcing directly through the operator structure of the KFE, and enable parallel-in-time training via a Lagrangian Galerkin projection and operator splitting. We evaluate NKEs on several stochastic benchmarks, including systems with coupled noise and jump processes, and verify that NKEs provide flexible models that accurately recover deterministic and stochastic dynamics with competitive predictive accuracy and improved training efficiency. Code and pretrained models will be released.
\end{abstract}

\section{Introduction}

% Stochastic differential equations (SDEs) are a canonical framework 
% for modelling continuous-time Markov systems, with applications spanning science, engineering, and finance. In this formulation, system trajectories evolve under the combined influence of a deterministic drift, representing the action of physical forces and system/market trends, and a stochastic driver, representing the effect of environmental noise or unknown/unmodelled effects. In real-world systems, these stochastic terms may take a variety of forms, going from continuous, random-walk-like noise to heavy-tailed perturbations, sudden shocks, and other forms of non-smooth stochastic forcing.

Stochastic differential equations (SDEs) are a canonical framework for modelling continuous-time Markov systems, with applications spanning science, engineering, and finance. In this formulation, system trajectories evolve under the combined influence of a deterministic drift, representing forces and systemic trends, and a stochastic driver, accounting for uncertainty and unresolved dynamics. While classical approaches often assume continuous perturbations, such as Brownian motion, real systems frequently exhibit heavy-tailed fluctuations, abrupt discontinuities, and intermittent shocks, as naturally captured by the family of Lévy processes.
% This generality is essential in practice: real-world systems are subject to noise that ranges from smooth perturbations in well-controlled physical settings to large discontinuous shocks in financial markets and biological systems.

Neural stochastic differential equations extend this framework by parameterizing deterministic and stochastic terms with neural networks. Introduced as stochastic generalizations of Neural Ordinary Differential Equations (Neural ODEs \cite{chen2018neural}),  Neural SDEs enable the learning of stochastic dynamics directly from data,
% . Instead of specifying the governing equations a priori, the drift and diffusion functions are learned from observations,
allowing flexible modelling in the absence of governing equations. 
% This approach has several advantages. 
They provide a principled way to incorporate uncertainty into continuous-time models, naturally capture irregularly sampled time series, and enable generative modelling of stochastic trajectories. These properties have led to growing applications across scientific \cite{rackauckas2020universal} and financial \cite{cohen2023arbitrage} modelling. However, existing implementations are still limited in their \textit{noise diversity} and \textit{scalability in time}.

Neural SDEs are generally limited to scalar or uncoupled Gaussian noise, due to the limitations of maximum-likelihood estimation, preventing them from modelling non-isotropic and discontinuous stochasticity. 
% Still, autoregressive frameworks exist for SDE learning with exotic noise [?]
% , such as Neural Jump SDEs \cite{jia2019neural} for discrete jumps and Levy Neural SDEs for $\alpha-$stable processes[?,?]. 
% \textbf{Fixed Noise Parametrization.} 
Moreover, existing Neural SDE models typically assume that the structure of the driving noise process is known, i.e., that the stochastic forcing follows a fixed parametric  form such as Brownian motion,
% an $\alpha-$stable process with known $\alpha$, 
or jumps with a predefined distribution. 
% In practice, however, the  true noise structure is rarely known a priori and may exhibit complex behavior and/or state-depedent noise. As a result, learning SDEs driven by an unknown noise process remains an important open problem.

% \vspace{-0.6cm}

Moreover, unlike Neural ODEs, Neural SDEs cannot be trained by directly differentiating observed trajectories, due to the presence of noise.
% makes sample paths non-differentiable and  non-unique, preventing straightforward gradient-based training. 
As a result, most approaches rely on repeatedly simulating trajectories using stochastic integration schemes. This procedure is computationally expensive,  difficult to parallelize in time, and often numerically inaccurate. 
% In the generative context, these limitations led to the development of Diffusion [?] and Flow Matching [?] models
% , which directly learn conditional maps between distributions. 
% Still, these frameworks learn mappings between two static distributions for sampling purposes, and thus cannot be immediately adapted to continuous-time sequential processes.
% Training these advanced frameworks requires balancing numerical stability with computational cost.[14] Rough Path Theory has emerged as a key tool for deriving efficient adjoint gradients for non-semimartingale and rougher noises \cite{kidger2021efficient}.[15]
% \textbf{Ill-posedness.}
Recent approaches attempt to alleviate these issues by leveraging Euler-Maruyama (EM) discretizations to enable simulation-free or partially parallelized training.
% Variations of this idea have been used to develop simulation-free frameworks for Latent SDEs (SDE Matching) and continuous-time sequential flow-matching [flow matching].
% bypass the need for expensive numerical integration during training, enabling the scaling of latent SDEs to high-dimensional datasets \cite{vargas2024sde}.[16, 17]
However, these discretization-based strategies may suffer from
 % However, as discussed in Section \ref{sec:splitting}, 
 ill-posedness, ultimately degrading predictive performance.
% \vspace{-0.8cm}

In this paper, we address these limitations by leveraging a deterministic representation  of stochastic dynamics via the  \textit{Kolmogorov Forward Equation} (KFE). While SDEs describe the \textit{microscopic} evolution of individual sample paths, the KFE provides a \textit{macroscopic} description of the system in terms of 
the evolution of a path's probability density (see Fig. \ref{fig:kolmogorov-eqs}),
% : \R \times\R^n \to \R\) 
% . In this formulation, stochastic dynamics are represented as a deterministic
given by a partial differential equation (PDE).
% governing the time evolution of probability distributions. 
The KFE therefore provides a deterministic representation of SDEs, enabling the direct application of numerical and theoretical tools from the analysis of PDEs. Most importantly, it is \textit{noise-agnostic}, which will allow us to represent substantially more diverse stochastic drivers.

\begin{figure}[h!]
    \centering
    \begin{subfigure}{0.425\textwidth}
        \centering
        \includegraphics[width=0.85\linewidth]{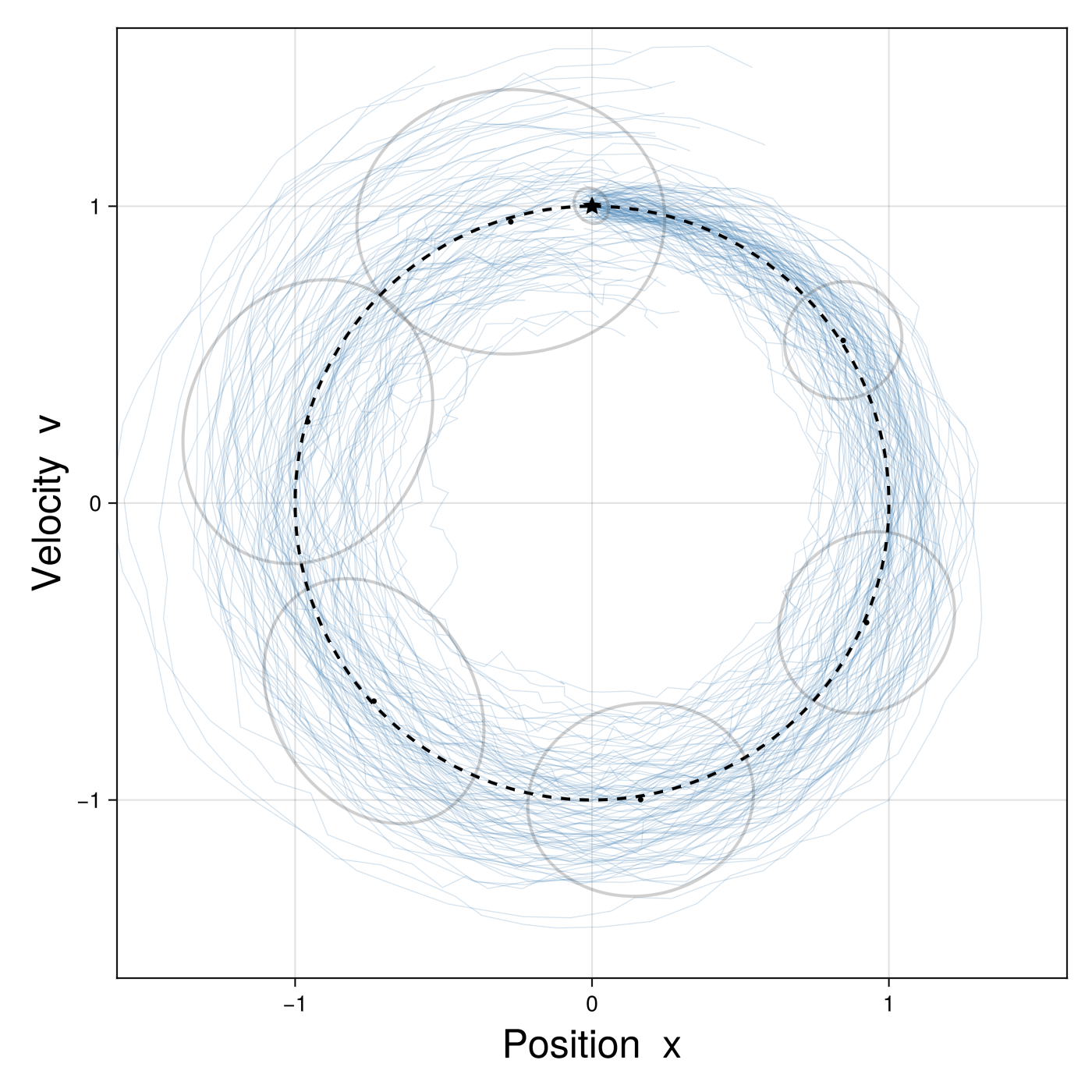}
        % \caption{}
    \end{subfigure}
    \hfill
    \begin{subfigure}{0.425\textwidth}
        \centering
        \includegraphics[width=0.98\linewidth]{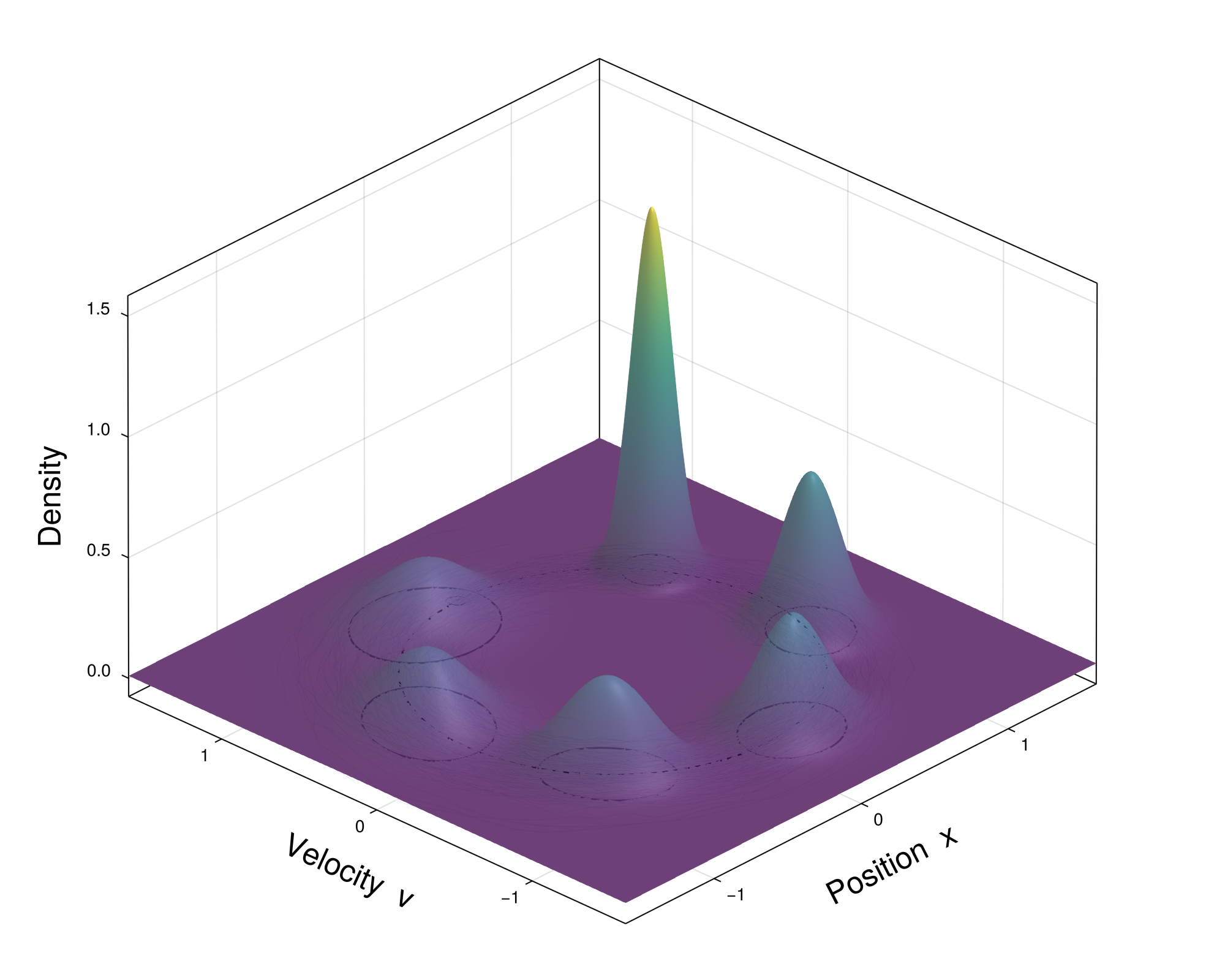}
        % \caption{}
    \end{subfigure}
    \caption{A stochastic harmonic oscillator. The Kolmogorov Equations treat individual particle realizations (left) from the macroscopic perspective of mass distributions moving (right).}
    \label{fig:kolmogorov-eqs}
\end{figure}

Inspired by this perspective, we propose Neural Kolmogorov Equations (NKEs), a framework for Neural SDE modeling based on directly  learning the advection and diffusion dynamics of the KFE. By modeling these operators with neural networks, NKEs provide a principled and accurate alternative to classical Euler–Maruyama–based training schemes, while
% enabling efficient parallelization across time. Importantly, this formulation allows the model not only to learn the drift dynamics but 
allowing the model to infer the structure of the driving noise process itself.
While the KFE has been used implicitly for SDE learning \cite{lai2023fpdiffusion}, to the best of our knowledge NKEs are the first framework to treat it as the primary learning object and parameterize both drift and stochastic forcing through its operator structure with neural networks. 
% Our contributions are summarized as follows:
% \begin{itemize}
%     \item \textbf{Scalability and accuracy.} We propose a novel learning scheme for drift and noise estimation based on a Neural PDE representation of the Kolmogorov equation. This formulation enables training procedures that avoid the expensive autoregressive simulation typically required in Neural SDE training while retaining accuracy.
%     We interpret Euler–Maruyama schemes as Lie-Trotter operator splitting,
%     % and their generalization to non-Gaussian noise, 
%     and introduce a deterministic generalization based on Strang splitting, which results in more accurate parallelizable learning.
%     \item \textbf{General-form noise learning.} We introduce a noise-agnostic learning strategy for additive stochastic processes based on  semigroup representations and generative modelling. NKEs directly model the  stochastic forcing of an SDE,  enabling a unified treatment of Gaussian/Brownian noise, heavy-tailed $\alpha-$stable noise, and Poisson jump processes with arbitrary jump distributions. Consequently, the framework can not only estimate  the intensity of pre-specified noise models but also  learn the structure of the driving noise itself when no parametric assumptions are available.
% \end{itemize}
\begin{enumerate}
\item \emph{General-form noise learning.} We introduce a noise-agnostic learning strategy for stochastic dynamics based on  semigroup representations and generative modelling. NKEs directly model the  stochastic forcing of an SDE,  enabling a unified treatment of Gaussian/Brownian noise with arbitrary coupling and jump processes with general jump distributions. Consequently, the framework can not only estimate  the intensity of pre-specified noise models but also  learn the structure of the driving noise itself when no parametric assumptions are available.

    \item \emph{Scalability in time and accuracy.} We propose a novel learning scheme for drift and noise estimation based on a Neural PDE representation of the Kolmogorov equation. This formulation enables training procedures that avoid the expensive autoregressive simulation typically required in Neural SDE training while retaining accuracy.
    We interpret Euler–Maruyama schemes as Lie-Trotter operator splitting,
    % and their generalization to non-Gaussian noise, 
    and introduce a deterministic generalization based on Strang splitting, which results in more accurate parallelizable learning.

\end{enumerate}
    
% NKEs exploit Lagrangian Galerkin methods and operator splitting 
% and other techniques from numerical methods for PDEs, 

% The framework of Kolmogorov Equations itself is not new and has equivalent probabilistic representations that have been implicitly exploited in multiple works on Neural SDEs [???]. 
% % While related ideas have appeared implicitly in prior work on Neural SDEs, they have not been treated as the primary object of learning.
% Instead, we believe to be the first to place the Kolmogorov forward operator at the center of the modeling and learning process, and to explicitly leverage its numerical structure. Our contributions are:

% We analyse parallelizable learning schemes from the perspective of operator-splitting techniques from numerical methods for PDEs. 

% while mitigating the ill-posedness of EM methods.

We evaluate the proposed framework on multiple benchmarks in the field of scientific machine learning, including systems driven by non-diagonal and discontinuous noise. We further compare NKEs with existing autoregressive and simulation-free frameworks for SDE learning.

% \textbf{Noise Learning.} We describe a novel strategy for noise learning, which enables the modelling of general driving processes including Brownian/Gaussian noise, Poisson point processes, and heavy-tailed $\alpha-$stable noise. By representing these processes as semigroups, NKEs can learn the parametrization of the noise itself, as opposed to relying on structural assumptions about it.

% We identify an issue of ill-posedness in Euler-Maruyama approaches to SDE-learning in the literature and mitigate it via a time-symmetrical learning scheme.

% \textbf{Accuracy.} We analyse parallelizable learning schemes from the perspective of operator-splitting techniques from numerical methods for PDEs. This allows for a precise description of EM schemes and their generalization to non-Gaussian noise, as well as the use of second-order Strang splitting, leading to substantially more accurate learning schemes than the EM discretizations present in the literature.
% Approaching the learning problem from the perspective of PDEs enables the usage of operator-splitting techniques, leading to an approach . Additionally, 

% They may also be combined with previous approaches to model state-dependent Gaussian noise.

% A detailed literature review may be found in Appendix \ref{sec:Related-Work}.

\section{Related Work}

\paragraph{Neural SDEs.}
Neural ODEs~\cite{chen2018neural} and Neural Controlled Differential Equations~\cite{kidger2020neural} learn continuous-time deterministic dynamics in terms of neural vector fields, which can be efficiently trained by directly regressing on estimated derivatives of time-series data. Neural SDEs extend this paradigm to stochastic dynamics by jointly modelling deterministic and stochastic forcing. Early formulations such as latent SDEs~\cite{li2020scalable} and Neural SDE GANs~\cite{kidger2021neural} relied on autoregressive simulation and stochastic adjoint methods, making training computationally expensive and difficult to parallelize in time.

\paragraph{Diffusion and Flow-Matching.}
In the generative setting, the limitations of autoregressive Neural-SDE-based Continuous Normalizing Flows~\cite{grathwohl2019ffjord} motivated the development of simulation-free approaches such as diffusion models~\cite{ho2020denoising}, score-based generative models~\cite{song2021score}, and Flow Matching~\cite{lipman2023flow}. While highly successful for generative sampling, these approaches primarily learn mappings between static distributions and are not naturally designed for continuous-time stochastic system identification from sequential observations.

\paragraph{Simulation-free learning and Euler--Maruyama methods.}
The scalability limitations of conventional Neural SDEs have motivated multiple simulation-free approaches based on Euler--Maruyama discretizations of Itô SDEs~\cite{kloeden1992numerical}. These methods directly regress on finite differences between consecutive observations, similarly to Neural ODE training. Examples include Trajectory Flow Matching~\cite{zhang2024trajectory}, which adapts flow-matching ideas to stochastic sequential processes, SDE Matching~\cite{bartosh2025sdematching}, which proposes simulation-free latent SDE training via matching objectives, and Unified Neural SDE frameworks~\cite{liu2019neuralsde}, which leverage discretized stochastic dynamics for scalable learning. These methods typically rely on separable drift and diffusion estimation under Gaussian assumptions.

\paragraph{Jump-SDEs.}
Many systems in science, engineering and finance exhibit stochastic dynamics beyond continuous Gaussian perturbations. Neural Jump SDEs~\cite{jia2019neuraljump} and Neural Stochastic Temporal Jump Processes~\cite{zhang2024neuraljumpdiffusion} study discontinuous dynamics in the context of temporal point processes and event modelling. Neural Jump ODEs~\cite{rubanova2019latent} instead model discontinuities arising from irregular observations and filtering problems. More recently, Neural SDEs with jumps~\cite{rong2026neuralsdejumps} have incorporated Poisson jump processes into neural stochastic dynamics. However, existing approaches generally assume fixed jump families or focus on event-driven dynamics, rather than learning general jump-diffusion operators directly from density evolution.

\paragraph{Neural Fokker--Planck models.}
Neural approximations to Fokker--Planck equations have been proposed using normalizing flows~\cite{rezende2015variational} and Physics-Informed Neural Networks~\cite{raissi2019physics}. To the best of our knowledge, these approaches do not learn stochastic operators directly from trajectory data and are generally limited to continuous diffusion dynamics.

\paragraph{Generator learning.}
The classical machinery of kernel methods and composition operators has also been used for stochastic generator learning. In particular, Extended Dynamic Mode Decomposition~\cite{williams2015data} and related Koopman-based methods approximate backward generators and transfer operators from data. These approaches may be viewed as adjoint to ours, but typically learn combined operators implicitly and do not directly parameterize advection, diffusion and jump terms with neural representations.

\paragraph{Projection filtering and particle methods.}
Projecting stochastic dynamics onto Gaussian families is a classical idea in nonlinear filtering and stochastic simulation. Projection filtering methods~\cite{brigo1998differential,brigo1999approximate,hanzon1991projection} provide the theoretical background for finite-dimensional approximations of infinite-dimensional stochastic dynamics. Related moment-closure methods~\cite{dipaola2002approximate} and Smoothed Particle Hydrodynamics (SPH)~\cite{gingold1977sph,lucy1977numerical} similarly approximate evolving densities through moving localized kernels. Our work builds on these ideas to construct neural operator models on projected Gaussian-mixture representations.

\newpage 

\begin{figure*}[h!]\label{fig:kolmogorov-gaussians}
    \centering
    \includegraphics[width=0.8\linewidth]{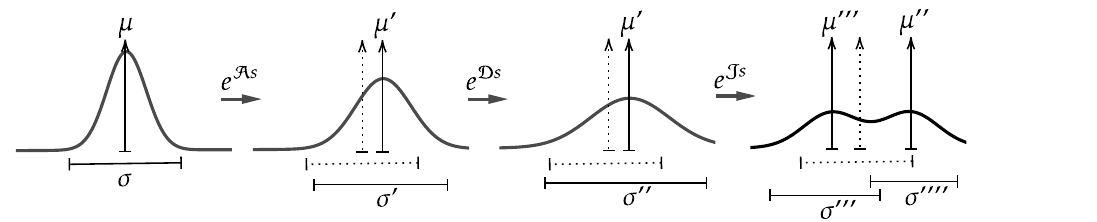}
    \caption{The forward Kolmogorov semigroup $e^\L$ on a Gaussian. Advection transports its mean and stretches its covariance, diffusion spreads its covariance, and the jump semigroup splits it.}
\end{figure*}

 \section{Problem formulation}

We consider 
% Lévy-Itô semimartingales, the most 
fully general continuous-time Markov dynamics with infinitely divisible increments; by the \emph{L\'{e}vy--Khintchine theorem}, they take the following canonical representation as a Lévy-Itô SDE:
% \begin{equation} \label{eq:sde}
%     \mathrm{d}x_t = f(x_t)\,\mathrm{d}t + g(x_t)\mathrm{d}w_t + ,
% \end{equation}
% \begin{equation} \label{eq:sde}
%     \mathrm{d}x_t = f(x_t)\,\mathrm{d}t 
%     + g(x_t)\,\mathrm{d}w_t 
%     + \int_{\mathbb{R}^d} h(x_{t-},z)\,\tilde{N}(\mathrm{d}t,\mathrm{d}z),
% \end{equation}
\begin{equation} \label{eq:sde}
    \mathrm{d}x_t = f(x_t)\,\mathrm{d}t 
    + g(x_t)\,\mathrm{d}w_t 
    + \int_{\mathbb{R}^d} z\,J(\mathrm{d}t,\mathrm{d}z),
\end{equation}
given in terms of a time-indexed state variable \(x_t \in \R^d\) and a Lévy triplet $(f,g,h)$:
 \begin{itemize}
    \item 
     The \emph{drift} $f: \R^d \to \R^d$, which models smooth, deterministic dynamics in terms of forces;
     
     \item 
     The \emph{noise coefficient} $g: \R^d \to \R^{d\times d'}$, which models continuous, small-scale random fluctuations driven by a Brownian motion $w_t \in \R^{d'}$;

     \item
     The jump term $ z\,J(\mathrm{d}t,\mathrm{d}z)$, that captures \emph{discontinuous dynamics} induced by a compound jump process $J$ with a state-dependent Lévy measure $h:\R^d \to \mathcal{M}(\R^d)$. 
\end{itemize}
  % This means an architecture that can solve \eqref{eq:sde} loses no generality over any alternative SDE representation.
% \begin{enumerate}
%     \item \textbf{(A1)}~the SDE is autonomous;
%     \item \textbf{(A2)}~jumps are compensated and have finite moments.
% \end{enumerate}

Our objective is to solve the following inverse problem: Given $\{{x}^{(m)}_{n}\}_{n=1 \, m=1}^{N\quad M}$, a sequence of $M$ independent realizations of \eqref{eq:sde} observed over $N$ uniformly distributed discrete \textit{time steps} $0, s, 2s, \dots, Ns$,
% Our objective is to solve the following inverse problem: Given $\{(x^{(n)},y^{(n)})\}_{n=1}^N$, a sequence of $N$ observed increments over a time-step $h$, i.e. $y^{(n)} = x^{(n)}_{0+h}$,
% \footnote{Unless stated otherwise, we will consider a sampling interval $h=1$ for simplicity.
% Note that the methods we describe here can be easily translated to other sampling rates, including non-uniform sampling
% }, 
learn a drift
$f_\theta:\mathbb{R}^d\to\mathbb{R}^d$,
a diffusion coefficient
$gg^\top_\vartheta:\mathbb{R}^d\to\mathbb{R}^{d\times d}$,
and a conditional generative model $h_\Theta :\mathbb{R}^d\to\mathcal{M}$
% $q_\phi(z\mid x)$ 
for the jump distribution, where $\mathcal{M} := \mathcal{M}(\R^d)$ is the space of measures in $\R^d$.
% , along with an estimated frequency $\Tilde\lambda$.
We make the following assumptions: \textbf{(A1)}~all terms are continuously differentiable and \textbf{(A2)} uniformly Lipschitz continuous and \textbf{(A3)}~all jump measures have bounded moments.
% such that the semigroup induced by=$(f_\theta, gg^\top_\vartheta, q_\phi)$ approximates the true transition kernel $p(x'\mid x)$ of the data-generating SDE.

\section{The Kolmogorov Forward Equation}

The \emph{Kolmogorov Forward Equation},
also known as the Fokker-Planck Equation,
is a linear PDE that describes the 
% time evolution of the probability density function of a particle driven by a Lévy SDE; 
% it may be interpreted as viewing
evolution of a Lévy SDE at the level of moving probability masses, with individual particles represented as Dirac delta measures and ensembles represented as smooth densities. Take the $d-$dimensional SDE in eq. \ref{eq:sde}. The KFE models the probability $p_t(x)$ of finding a particle at a given position $x \in \R^d$ at time $t$. Given a probability measure $p_0 \in \mathcal{M}$ for the initial position of the particles, the KFE describes the evolution $p_t \in \mathcal{M}$ in terms of a PDE in $d$ spatial dimensions, given by the Kolmogorov Forward \textit{generator} $\mathcal{L} : \mathcal{M} \mapsto \mathcal{M}$ 
% \footnote{We follow the literature and denote the forward operator as the adjoint of the backward generator $\mathcal{L}$.}
:
% \begin{equation}
%     \frac{\d}{\d t} p_t(x) = \mathcal{L}^*p_t = (\A + \D + \J)p_t. 
% \end{equation}
% \begin{equation}
%     \frac{\d}{\d t} p_t(x) = \mathcal{L}p_t = (\mathcal{L}^A + \mathcal{L}^D + \mathcal{L}^J)p_t. 
% \end{equation}
\begin{equation}
    \frac{\d}{\d t} p_t(x) = \mathcal{L}p_t = (\mathcal{A} + \mathcal{D} + \mathcal{J})p_t. 
\end{equation}
$\L$ describes the infinitesimal time variation of $p_t$; again by the \textit{Lévy-Khintchine Theorem}, it may be decomposed similarly to \ref{eq:sde} into an \textit{advection} term $\A$, a \textit{diffusion} term $\D$, and a \textit{jump} term $\J$:
% \begin{equation}
% \begin{aligned}
% \A p_t(x) &= -\nabla \cdot (f(x)p_t(x))
% && \text{(advection)} \\
% \D p_t(x) &= \nabla^2 : ( gg^\top(x)   p_t(x))
% && \text{(diffusion)} \\
% \J p_t(x) &= \lambda \int_{\mathbb{R}^d} \big[p_t(x - y) - p_t(x)\big] \,h(x,\d y)
% && \text{(jump)}
% \end{aligned}
% \end{equation}
% \begin{equation}
\begin{align}
\mathcal{A}\, p_t(x) &= -\nabla \cdot (f(x)p_t(x))
&& \text{(advection)} \\
\mathcal{D}\, p_t(x) &= \frac{1}{2}\nabla^2 : ( gg^\top(x)   p_t(x))
&& \text{(diffusion)} \\
% \mathcal{J} p_t(x) &= \int_{\mathbb{R}^d} \big[p_t(x - y) - p_t(x)\big] \,h(x,\d y)
% && \text{(jump)}
\mathcal{J} p_t(x) &= \int_{\mathbb{R}^d} p_t(x - y)h(x-y,y) - p_t(x)h(x,y) \, \d y
&& \text{(jump)} \label{eq:jump_op}
\end{align}
% \end{equation}
Here, $f : \mathbb{R}^d \to \mathbb{R}^d$ is the drift vector field driving deterministic transport, $gg^\top := g g^\top \in \R^{d \times d}$ is the diffusion tensor, and the jump term is governed by $h$.
The operators $\nabla$ and $\nabla^2$ denote the gradient and Hessian operators, with inner product $A \cdot B := \mathrm{Tr}(A^\top B)$ and $a:b = \sum_{i,j}a_{ij}b_{ij}$. 

The density at time $t$ is then given by the action of the \textit{semigroup} $e^{\L t}:\mathcal{M} \to \mathcal{M}$ via $p_t = e^{\L t}p_0$; likewise, the fixed-step increments in the data are given by the short-time iterate $p_{t+s} = e^{\L s}p_{t}$. The generator $\L$, on the other hand, is never directly observed and must be estimated.

\newpage 

\subsection{Lagrangian Galerkin Projection and Gaussian Mixtures}
The forward Kolmogorov generators and semigroups are infinite-dimensional objects; still, local probability masses generated by short-time stochastic dynamics are often approximately Gaussian. Inspired by Smoothed Particle Hydrodynamics, we will use a \textit{Lagrangian Galerkin discretization} to extract a tractable, finite-dimensional projection of the KFE which is amenable to learning. 

Lagrangian Galerkin methods approximate $\L$ via its action on a set of \textit{moving basis functions}; more specifically, we will use a basis of meshless \textit{Gaussian mixtures} to represent {our evolving densities}:
% (see Fig \ref{fig:kolmogorov-gaussians}):
% via their associated \textit{fundamental solutions}, also known as Green functions of transition densities.
% 
% \begin{equation}
%     G_{\mu,\Sigma}(x) := 
% \frac{1}{(2\pi)^{d/2} \, |\Sigma|^{1/2}}
% e^{\!
% -\frac{1}{2}(x - \mu)^\top \Sigma^{-1} (x - \mu)
% }, \quad  p_t(x) = \sum_k w^{(k)}G_{\mu^{(k)}_t,\Sigma_t^{(k)}}(x).
% \end{equation}
\begin{equation}
%     G_{\mu,\Sigma} := 
% \frac{1}{(2\pi)^{d/2} \, |\Sigma|^{1/2}}
% e^{\!
% -\frac{1}{2}(x - \mu)^\top \Sigma^{-1} (x - \mu) }, \quad 
p_t(x) = \sum_{k=1}^K \pi^{(k)}G_{\mu^{(k)}_t,\Sigma_t^{(k)}}(x),
\end{equation}
% \begin{equation}
%     p_0(x) = \sum_k w^{(k)}N\left(x;\mu_0^{(k)},\Sigma_0^{(k)}\right) \implies p_t(x) = \sum_k w^{(k)}e^{\L t}N\left(x;\mu_t^{(k)},\Sigma_t^{(k)}\right),
% \end{equation}
where each $k-$indexed component $\pi G_{\mu,\Sigma}$ is a normal distribution parametrized by its weight $\pi \in \R$, its mean $\mu \in \R^d$ and its covariance $\Sigma \in \R^{d\times d}$. In practice, this representation will reduce the original infinite-dimensional evolution to a corresponding dynamics in the Gaussian parameters $[\pi,\mu,\Sigma]$, which NKEs will take as inputs to describe the KFE as moving, spreading, and splitting Gaussians. 
% NKEs approximate the action of $\L$ on these functions as they move, spread, and split
% These Gaussians represent large ensembles as moving mixtures, while individual particles may be represented in the limit $\varepsilon \rightarrow 0$ as delta measures $G_{\mu,\varepsilon I}$; in the data, they will be approximated from particle clusters. 
% by describing how their parameters change over time.

The main tool for this will be approximate closure of Gaussians under $\L$.  
% Advection/diffusion are decribed in terms of the approximate closedness of Gaussians under their action:
For sufficiently small $s$, the image of a localized Gaussian under the advection–diffusion semigroup is well approximated by a Gaussian with transported mean and covariance (see the Appendix for a rigorous statement):
% This will allow us to define dynamics on their projection on a Gaussian submanifold given by parameter space:
% arising from \ref{eq:advection-diffusion} via convolution with the jump measure:
\begin{equation}\label{eq:advection-diffusion}
            % e^({\Ah + \Dh)}G_{\mu,\Sigma} \approx G_{\mu + h\Delta \mu,\Sigma + h\Delta \Sigma}.
            e^{\mathcal{A}s + \mathcal{D}s}G_{\mu,\Sigma} \approx G_{\mu + s\Delta \mu,\Sigma + s\Delta \Sigma}.
            % \qquad e^{ \mathcal{L}^J}G_{\mu,\Sigma} \approx \sum_{k'} w^{(k')}N\left(x;\mu^{(k')},\Sigma^{(k')}\right),
\end{equation}
 In contrast, the image of a Gaussian under the jump semigroup is given by the convolution with the jump measure, and cannot generally be described with a single Gaussian. Still, if the jump semigroup $e^{\mathcal{J}s}$ is given by a mixture $H^s$, its action leads to another mixture, as per closure under convolution:
 % By modelling $h$ itself as a full Gaussian Mixture, we may obtain the representation of the jump semigroup from the closure of Gaussians under convolution.
 % In practice, jumps map a single Gaussian onto many.
 \begin{equation}
    e^{\J s}G_{\mu,\Sigma} = H^s * G_{\mu,\Sigma}, 
                % = \sum_{k'} \hat\pi^{(k')} G_{\hat\mu^{(k')},\hat\Sigma^{(k')}},
 \end{equation}
% because they form a basis, this could then be used to model the dynamics over any other suitable function.
% ; these are given in terms of three components:
We may formalize this intuition by defining the \textit{projection} of the dynamics onto the space of Gaussian Mixtures.
Let $\mathcal{G}^{(k)}$ 
be the space of K-element Gaussian mixtures, with $\mathfrak{g}\in\mathcal{G}^{(k)}
\Longleftrightarrow \mathfrak{g} = (\pi^{(k)},\mu^{(k)},\Sigma^{(k)})_{k=1}^K$. Let also  
 $\Pi_{K}: \mathcal{M} \to \mathcal{G}^K$ and $\Pi_K^{\dagger}: \mathcal{G}^K \to \mathcal{M}$ be the projection and lifting operators between the space of measures and their corresponding mixture representation:
  \begin{equation}
    \Pi_{K}: p \to (\pi^{(k)},\mu^{(k)},\Sigma^{(k)})_{k}, 
     \quad \Pi_K^{\dagger}: (\pi^{(k)},\mu^{(k)},\Sigma^{(k)})_k \to \sum_{k=1}^K\pi^{(k)}G_{\mu^{(k)},\Sigma^{(k)}}.
 \end{equation}
 For $K=1$, this is equivalent to moment matching and can be directly calculated:
 \begin{equation}
     \Pi_{1} p_t = (1, \langle x, p_t \rangle, \langle (x-\mu)(x-\mu)^\top, p_t \rangle),
 \end{equation}
 % by taking inner products with tensor powers of $x$;
 where $\mu$ is the mean of $p_t$ and $1$ denotes unit mass. For larger $K$, regularization is required to ensure uniqueness and continuity (see Appendix). NKEs will learn the projected generator $\widehat\L: \mathcal{G}^{(k)} \to \mathcal{G}^{(k)}$: 
 \begin{equation}
     \widehat\L = \Pi \circ \L  \circ \Pi^\dagger.
 \end{equation}

\subsection{Advection Generator}
The action of the advection operator takes a simple form: it transports the mass under the Gaussian according to the drift $f$.
 The action of the projected advection operator $\widehat{\mathcal{A}} = \Pi_1 \circ \mathcal{A} \circ \Pi_1^{\dagger}$ follows classically from its adjoint $\A^\star$, which can be obtained via integration by parts:
 \begin{equation}
     \langle \phi, \A \varphi\rangle = -\int_{\R^d}\phi(x) \, [\nabla \cdot (f(x) \varphi (x))]\, \d x =
     \int_{\R^d} [f\cdot\nabla\phi]\, \varphi (x)\, \d x = \langle \A^\star\phi, \varphi\rangle 
 \end{equation}
For the mean, we may take the inner product with the test function $\phi = x$:
\begin{equation}
    \langle x, \A \,G_{\mu,\Sigma} \rangle =  \langle  \A^\star x, G_{\mu,\Sigma}(x) \rangle = \langle f(x)\, , G_{\mu,\Sigma} \rangle, 
\end{equation}
while for the covariance, we take the inner product with the test function $\phi = (x-\mu)(x-\mu)^\top$:
\begin{multline}
    \left\langle (x-\mu)(x-\mu)^\top, \A \, G_{\mu,\Sigma} \right\rangle
    =
    \left\langle \A^\star (x-\mu)(x-\mu)^\top, G_{\mu,\Sigma} \right\rangle
    \\ =
    \left\langle f(x)(x-\mu)^\top + (x-\mu) f(x)^\top, \,G_{\mu,\Sigma} \right\rangle .
\end{multline}
Combined, these expressions lead to full action of the projected advection on a Gaussian's parameters:
\begin{equation}
\widehat\A \, (\pi, \mu,\Sigma) = 
\Big(0, \, \langle f,  G_{\mu,\Sigma} \rangle, \,
\left\langle
(x-\mu)f(x)^\top + f(x)(x-\mu)^\top, G_{\mu,\Sigma}
\right\rangle \Big)
 \end{equation}
Intuitively, we have that the mean is transported according to a smoothed vector field and the covariance is stretched and rotated by its divergence. Its total mass does not change.
% \begin{figure*}[h!]\label{fig:kolmogorov-gaussians}
%     \centering
%     \includegraphics[width=0.99\linewidth]{Images/diagram-adv-diff-jump2.pdf}
%     \caption{The forward Kolmogorov semigroup $e^\L$. Intuitively, the advection-diffusion semigroup acts on a Gaussian's mean and covariance, while the jump semigroup split it.}
% \end{figure*}
\subsection{Diffusion Generator}
The action of the diffusion operator likewise admits a simple interpretation: it spreads mass towards local equilibrium, possibly at space-dependent rates. Its projection
$\widehat{\mathcal{D}} = \Pi_{\{1\}} \circ \mathcal{D} \circ \Pi^\dagger_{\{1\}}$
follows analogously from its adjoint, obtained via integration by parts:
\begin{equation}
    \langle \phi, \D \varphi\rangle
    =
    \int_{\R^d}
    \phi(x)\,
    \left[\frac{1}{2}\nabla^2 : \left(gg^\top(x)\varphi(x)\right)\right]\,\d x
    =
    \int_{\R^d}
    \left[\frac{1}{2}\nabla^2 \phi(x) : gg^\top(x)\right]
    \varphi(x)\,\d x
    =
    \langle \D^\star\phi,\varphi\rangle .
\end{equation}
We may calculate its action on the mean by again testing against $\phi = x$:
\begin{equation}
    \langle x, \D G_{\mu,\Sigma} \rangle
    =
    \langle \D^\star x, G_{\mu,\Sigma}\rangle
    =
    0,
\end{equation}
since $\nabla^2 x = 0$. For the covariance, analogous derivation with
$\phi = (x-\mu)(x-\mu)^\top$ leads to:
\begin{multline}
    \left\langle
    (x-\mu)(x-\mu)^\top,
    \D G_{\mu,\Sigma}
    \right\rangle
    =
    \left\langle
    \D^\star
    (x-\mu)(x-\mu)^\top,
    G_{\mu,\Sigma}
    \right\rangle
    =
    \left\langle
    gg^\top(x),
    G_{\mu,\Sigma}
    \right\rangle,
\end{multline}
leading to the interpretation of the projected diffusion operator as spreading out a Gaussian's support by adding to its covariance, without any effect on its mean or total mass:
\begin{equation}
\widehat\D(\pi,\mu,\Sigma)
=
\left(
0, \,
0, \,
\left\langle
gg^\top,
G_{\mu,\Sigma}
\right\rangle
\right).
\end{equation}

\subsection{Jump Semigroup}

Unlike advection and diffusion, the action of the jump generator on a Gaussian cannot be generally described in terms of a new Gaussian with perturbed mean and covariance. For small jumps $y\leq\epsilon$ with smooth distributions, a Taylor expansion yields that the cumulative effect of $\J_{<\epsilon}$ is well approximated by an state-dependent advection-diffusion given by the first two moments of $h$(see the appendix):
\begin{equation}\label{eq:small-jump-taylor}
% \text{(Small jumps)} \quad
\J_{<\epsilon} p_t 
% = \int_{\|y\|<\epsilon}\!\!\big[p_t(x-y)-p_t(x)\big]h(x,dy)
\approx - \nabla\cdot (\beta(x) \, p_t(x)) + 
\frac12 \nabla^2:\big(\gamma(x) \,p_t(x)\big),
% \text{ where } Q_\varepsilon := \int_{\|y\|<\varepsilon} yy^\top h(dy),
\end{equation}
where $\beta$ and $\gamma$ stand for the mean and covariance of $h$.
For large jumps, however, this no longer holds. Instead, we will model the full jump semigroup over the time span $s$. For state-independent $h$, this may be written in terms of the finite-time Poisson distribution $H ^s$:
\begin{equation}
    e^{\mathcal{J}s}p_t = H^s * p_t, \, \text{ where } H ^s :=  e^{-\lambda s}\sum_{i=0}^\infty \frac{s^i h^{*i}}{i!} 
    % \text{ and } \lambda = \int_{\R^d} h(dy),
\end{equation}
where $\lambda$ is the total mass of $h$, and $h^{*n}$ denotes $n$-fold iterated convolution, with
$h^{\ast 0}=\delta_0$ corresponding to trajectories that do not
jump during the interval $[t,t+s]$.
For $H^s$ also a $K$-Gaussian mixture $(\alpha^{(k)},\beta^{(k)},\gamma^{(k)})_k$, the projected action of the jump semigroup on a Gaussian may be succinctly described in terms of the convolution between two mixtures:
\begin{equation}
    e^{\widehat \J s} (\pi, \mu,\Sigma) = \left(\alpha^{(k)},\beta^{(k)},\gamma^{(k)}\right)_k
*
\left(\pi,\mu,\Sigma\right)
=
\left(
\alpha^{(k)} \pi,\,
\beta^{(k)}+\mu,\,
\gamma^{(k)}+\Sigma
\right)_{k}. 
\end{equation}
This may be visualized in a simple context as follows: If particles have, say, a 10\% chance to jump right by one unit, the finite-time distribution will be modelled as two Gaussians, one with 90\% mass at the origin, and another with 10\% mass one unit to the right.

% State-dependent jumps $h(x,\d y)$, however, introduce two substantial complications, which will break this convolution structure. First is the two differing jump rates \textit{into} and \textit{out of} $p$, caused by the terms $h(x-y,y)$ and $h(x,y)$ in eq. \ref{eq:jump_op}.

If the jump distribution is instead space-dependent, we may instead consider $H^s(x)$ to have spatially varying parameters $[\alpha^{(k)}(x),\beta^{(k)}(x),\gamma^{(k)}(x)]_k$.
 However, this breaks the convolution structure from above: different points underneath the Gaussian will be exposed to different jump distributions, leading to a largely intractable object.
 
 We may recover tractability by leveraging the rapid decay of concentrated Gaussians and the smoothness of the space-varying jump parameters. We assume that the jump measures under the integral are well approximated by their distribution at the mean, around which most of the integral is concentrated:\begin{equation}
    \int_{\mathbb{R}^d} G_{\mu,\Sigma}(x-y)h(x-y,y) - G_{\mu,\Sigma}(x)h(x,\d y)
    \approx \int_{\mathbb{R}^d} \big[G_{\mu,\Sigma}(x - y) - G_{\mu,\Sigma}(x)\big] \,h(\mu,\d y)
\end{equation}
This will allow us to approximate the projected semigroup as the convolution with the Gaussian mixture that represents the jumps at the mean $\mu$, $H^s(\mu) = [\alpha^{(k)}(\mu),\beta^{(k)}(\mu),\gamma^{(k)}(\mu)]_k $:
\begin{equation}
    e^{\widehat \J s} (\pi, \mu,\Sigma) \approx \left(\alpha^{(k)}(\mu),\beta^{(k)}(\mu),\gamma^{(k)}(\mu)\right)_k 
*
(\pi,\mu,\Sigma)
\end{equation}
For sharp $(\Sigma \rightarrow 0)$ Gaussians $p_t$ and slowly varying parameters $\alpha, \beta, \gamma$, this approximation is provably sharp (see the Appendix). 

\newpage

 \begin{figure*}[h!]\label{fig:kolmogorov-gaussians}
    \centering
    \includegraphics[width=1.0\linewidth]{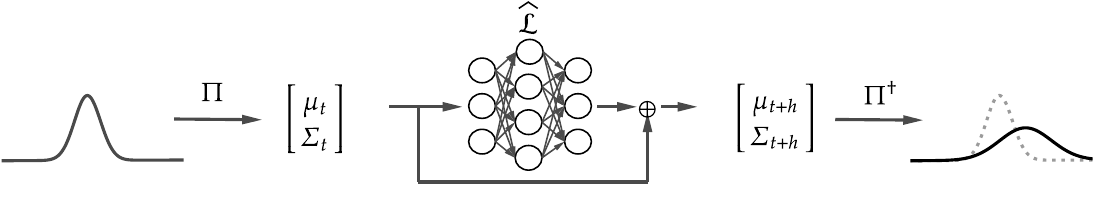}
    \caption{Neural Kolmogorov Equations. NKEs operate as deterministic Neural ODEs on the finite-dimensional (projected) space of Gaussian mixtures and model flows on their mean and covariance.}
\end{figure*}

\section{Neural Kolmogorov Equations}
% 
% We will model these objects via their action on a set of \textit{basis functions};
% \subsection{Lagrangian Galerkin formulation}

\emph{Neural Kolmogorov Equations} describe SDEs by modelling the associated advection, diffusion, and jump operators in the Kolmogorov formalism; more specifically,
NKEs learn the \textit{projection} of the forward operators on the space of Gaussians, based on the dynamics imposed by the projected Kolmogorov equation on the Gaussian-projected density $\widehat p$:
\begin{equation}
    \frac{\mathrm{d}}{\mathrm{d}t}
% \begin{pmatrix} \mu \\ \Sigma \end{pmatrix}_t
% (\pi,\mu,\Sigma)
\widehat p
= \widehat \L \, \widehat p
% \begin{pmatrix} \mu \\ \Sigma \end{pmatrix}_t
% (\pi,\mu,\Sigma)
% , \quad 
% \begin{pmatrix} \mu \\ \Sigma \end{pmatrix}_{t+s}
% = e^{\widehat \L s} \begin{pmatrix} \mu \\ \Sigma \end{pmatrix}_t
\end{equation}
NKEs are composed of three parts, \textit{advection}, \textit{diffusion} and \textit{jump}, which are trained in four steps:

\textbf{1. Cluster particles into approximate Gaussians and separate jumps.} The trajectory data, given in terms of particles, is clustered into empirical/weak approximations for Gaussians. Large jumps are filtered out via a thresholding procedure to allow for separate treatments for the advection-diffusion generator and the jump semigroup. 
 
 % \textbf{Advection ${\widehat{\mathcal L}}^A_{\theta}
\textbf{2. Learn the projected advection generator $\widehat\A_{\theta}
 % : \R^{d + d\times d} \to \R^{d + d\times d}
: \mathcal{G} \to \mathcal{G} $ from mean evolution.}
 NKEs approximate the lifted advection generator in terms of a neural drift $f_\theta: \R^d \to \R^d$, which plays an analogous role to the drift in Neural SDEs. $\widehat\A_\theta$ transports single localized Gaussians, and is trained by tracking the movement of their means $\mu$ over time.
 
 % \textbf{Diffusion $\widehat{\mathcal L}^*_D_\vartheta$.}
 % \textbf{Diffusion $\widehat\D_\vartheta
  \textbf{3. Learn the projected diffusion generator $\widehat\D_\vartheta: \mathcal{G} \to \mathcal{G}
 % : \R^{d + d\times d} \to \R^{d + d\times d}
 $ from covariance evolution.}
NKEs approximate the lifted diffusion generator as parametrized in terms of a (squared) noise coefficient $gg^\top_\vartheta: \R^{d} \to \R^{d\times d}$, analogous to the noise coefficient in Neural SDEs. $\widehat\D_\vartheta$ acts on single localized Gaussians via spreading, and is trained by tracking the time variation of their covariances.

  % \textbf{Jump $e^(\widehat{\mathcal L}^*_J_\phi)$.}
  % \textbf{Jump $e^({\J})_\phi
  \textbf{4. Learn the projected jump semigroup $\exp({\widehat\J s)}_\Theta : \mathcal{G} \to \mathcal{G}
  % : \R^{d + d\times d} \to \R^{(d + d\times d)\times k'}
  $ from filtered-out jump increments.}
  $e^{\widehat\J s}$ is approximated in terms of a conditional generative model 
  $H^s_\Theta: \R^{d} \to (\R^{d + d\times d})^k$ 
  that outputs the parameters for a Gaussian mixture modelling the short-time jump distribution at each Gaussian center $\mu$. It splits Gaussians via convolution, and is trained on filtered-out jump residues.

  \subsection{Operator splitting}
The basis of our strategy will be the interpretation of separable learning as \textit{operator splitting}. 
We observe that Euler-Maruyama approximations -- which consist on a drift step, followed by the addition of Gaussian noise -- correspond to first-order Lie-Trotter \cite{trotter1959product} splitting for $\L = \A+\D+\J$:
\begin{equation}
    \text{(Lie-Trotter)  }\;\quad
 e^{\L s} \approx \, 
 e^{\J s} \, 
 e^{\D s}\, 
 e^{\A s}.
 % + \mathcal{O}(h^2).\\
\end{equation}
In contrast, NKEs are based on a Strang-type splitting \cite{strang1968construction}, second-order in the absence of jumps:
% \begin{align}
% \text{Lie-Trotter/Euler-Maruyama:}\quad
% e^{h(\mathcal{L}_A + \mathcal{L}_B + \mathcal{L}_D)}
% &= e^{h\mathcal{L}_J}\, e^{h\mathcal{L}_D}\, e^{h\mathcal{L}_A}
% + \mathcal{O}(h^2), \\[0.5em]
% \text{Strang/Kolmogorov:}\quad
% e^{h(\mathcal{L}_A + \mathcal{L}_D + \mathcal{L}_J)}
% &= e^{\frac{h}{2}\mathcal{L}_A}
% \, e^{\frac{h}{2}\mathcal{L}_D}
% \, e^{h\mathcal{L}_J}
% \, e^{\frac{h}{2}\mathcal{L}_D}
% \, e^{\frac{h}{2}\mathcal{L}_A}
% + \mathcal{O}(h^3).
% \end{align}
\begin{align}
\text{(Strang)}\quad
e^{\L s} \approx
& \, e^{\A s/{2}}
\, e^{\J s}
\, e^{\D s}
\, e^{\A s/{2}}.
% + \mathcal{O}(h^3).
\end{align}
This scheme instead resembles higher-order numerical schemes, such as Itô-Heun, composed of half-advection steps before and after the addition of noise. 

This will allow us to conceive more accurate loss functions by leveraging the invertibility of the advection semigroup, leading to the use of second-order trapezoidal differences in the data.
% and inner products against test functions:
% \begin{equation}
% \label{eq:half-step}
% e^{-\A s / 2}\, p_{t+1}
% \;\approx\;
% e^{\D s}e^{\A s / 2}\, p_{t},
% \end{equation}
These can then be combined with inner products to obtain the dynamics of the moments of Gaussians:
\begin{equation}
\label{eq:half-step}
\langle \phi(x), e^{-\A s / 2}\, p_{t+1}(x) \rangle 
\;\approx\;
\langle \phi(x), e^{\J s }e^{\D s }e^{\A s / 2}\, p_{t}\rangle.
\end{equation}
% , \quad \langle (x-\mu)(x-\mu)^\top, e^{-\A s / 2}\, p_{t+1}(x) \rangle \\
% \;\approx\;
% \langle (x-\mu)(x-\mu)^\top, e^\D e^{\A s / 2}\, p_{t}\rangle. 
In fact, operator splitting will allow us to apply and train each component separately.

\newpage 

\subsection{ Gaussian approximation and jump thresholding}

\subsubsection{Gaussian approximation}
  In practice, we do not directly observe the evolution of Gaussian densities during training. Instead, they will be represented empirically as particle clusters: 
   we cluster the particles $\{x_n^{(m)}\}$ into $C$ clusters $\{\Tilde G^{(c)}_n\}_{c,n}$ with $\{\Tilde\mu^{(c)}_n\}$ and $\{\Tilde\Sigma^{(c)}_n\}$ as means and covariances respectively. 
   
   The clusters can be non-exclusive, and are built from the $\kappa_{nn}$ nearest neighbors of selected particles, leading to weak approximations to an element in our basis. We then track the changes in mean and covariance of these clusters between observations to estimate their time derivative:
\begin{equation}
    \Tilde\mu^{(c)}_n = \frac{1}{\kappa_{nn}}\sum_{x_n^{(m)} \in \Tilde G^{(c)}_n} \xx m n, \quad \Tilde\Sigma^{(c)}_n = \frac{1}{\kappa_{nn}}\sum_{x_n^{(m)} \in \Tilde G^{(c)}_n} (x^{(m)}_n - \Tilde\mu^{(c)}_n)(x^{(m)}_n - \Tilde\mu^{(c)}_n)^\top
\end{equation}
These clusters act as Dirac delta ensembles, which will appear as discrete sums under inner products.
\begin{equation}
    % \C_n^{(c)} = \sum_{\x^{(m)}_n\in \Tilde N^{(c)}_n} \delta_{x^{(m)}_n}, \, \implies \, \langle a(x),  \Tilde N ^{(c)}_n(x)\rangle 
    % = \sum_{x^{(m)}_n\in \Tilde N^{(c)}_n} a(x^{(m)}_n).
    \Tilde G_n^{(c)} = \sum_{x^{(m)}_n\in \Tilde{G}^{(c)}_n} \delta_{x^{(m)}_n}, \, \implies \, \langle \phi(x),  \Tilde G ^{(c)}_n(x)\rangle 
    = \sum_{x^{(m)}_n\in \Tilde{G}^{(c)}_n} \phi(x^{(m)}_n).
\end{equation}
Intuitively, this strategy may be seen as using these Gaussian-like clusters as local denoised estimators for the dynamics, so that groups of particles move in an approximately deterministic manner.

\subsubsection{Jump thresholding}
% We combine splitting with jump thresholding, which will allow us to deal with jumps separately by filtering out large increments. This allow us to separate ordinary advection-diffusion from jumps, allowing for each component to be learned separately and more accurately.    

% As opposed to advection-diffusion, the action of the generator $\J$ on regular Gaussians cannot be represented in a simple manner. 
We will treat jumps separately via thresholding.
% to afterwards model their semigroup directly via the jump distribution and its GM representation. 
Under sufficiently small timesteps, diffusion-driven increments concentrate around a characteristic local scale, while jump events induce comparatively large displacements. 

Let $\Delta x_n^{(m)} := \|x_{n+1}^{(m)} - x_n^{(m)}\|$ denote the magnitude
of the observed increment. We flag each increment as jump-contaminated
when it exceeds the typical scale of its peers at the same time step, as
measured by the median absolute deviation (MAD)
\cite{rousseeuw1993alternatives}. Writing $\med_{m}$ for the median across $m$-indexed particles, we define
\begin{equation}
    \chi_n^{(m)} \;:=\; \mathbf{1}\!\left\{
        \Delta x_n^{(m)} - \med_{m}\!\Delta x_n^{(m)}
        \;>\;
        \kappa_{\text{sens}}\kappa_{\text{Gaussian}} \cdot
        \med_{m''}\!\bigl|\Delta x_n^{m''} - \med_{m'}\!\Delta x_n^{{m'}}\bigr|
    \right\},
\end{equation}
with $\kappa_{\text{sens}} > 0$ controlling the detector sensitivity, with
$\kappa_{\text{Gaussian}} = 1/\Phi^{-1}(0.75) \approx 1.4826$ the standard Gaussian
calibration. Detected jumps ($\chi=1$) are excluded when learning the
advection--diffusion generator and retained when learning the jump semigroup.

% We therefore consider the transition magnitudes $d$ and estimate their typical diffusive scale using the median absolute deviation (MAD):
% \[
% m_t = \mathrm{median}_n(d_n^t),
% \quad
% \hat{\sigma}_t
% =
% 1.4826\,
% \mathrm{median}_n\!\left(
% |d_n^t - m_t|
% \right),
% \quad d_n^t = \|x_{n+1}^{(m)} - x_n^{(m)}\|,
% \]
% where the factor \(1.4826\) ensures consistency with the standard deviation under Gaussian fluctuations \cite{hampel1974influence,rousseeuw1993alternatives}.
% Transitions are then classified according to the binary criterion
% \[
% \chi_n^t
% =
% \mathbf{1}\!\left[
% d_n^t > m_t + \kappa \hat{\sigma}_t
% \right],
% \]
% where \(\kappa > 0\) controls the detector sensitivity. We interpret \(\chi_n^t = 1\) as a likely jump-contaminated transition and \(\chi_n^t = 0\) as diffusion-dominated. These jumps will then be filtered out when learning the advection-diffusion generator, but included in the clusters when learning jumps.
% Since the MAD is median-based, it is substantially less sensitive to outliers and heavy-tailed contamination than variance-based estimators, making it particularly suitable for robust separation of sparse jump events from dominant diffusive fluctuations \cite{hampel1974influence,rousseeuw1993alternatives}. 
% We now extend the Lagrangian operator learning framework to the nonlocal component of the generator. While advection and diffusion are identified through the evolution of local first- and second-order moments of moving Gaussian probes, jumps manifest as nonlocal mass transfer that cannot be captured by these local statistics.

In the Kolmogorov picture, this jump thresholding implies a decomposition of the Lévy measure $h$. For any cutoff $\varepsilon>0$, the Lévy measure admits the canonical decomposition
\(
h=h_{<\varepsilon}+h_{\ge\varepsilon},
\)
separating small and large jumps and separating the generator accordingly with
\(
\J=\J_{<\varepsilon} + \J_{\ge\varepsilon}\).
%  A Taylor expansion yields that for mild jumps and small $\epsilon$, the cumulative effect of $\L^{J<\varepsilon}$ is well approximated by an effective state-dependent diffusion:
% \begin{equation}\label{eq:small-jump-taylor}
% \int_{\|y\|<\varepsilon}\!\!\big[\phi(x+y)-\phi(x)\big]h(dy)
% \approx
% \frac12 \operatorname{tr}\!\big(Q_\varepsilon \nabla^2\phi(x)\big), \text{ where } Q_\varepsilon := \int_{\|y\|<\varepsilon} yy^\top h(dy),
% \end{equation}
% with covariance $Q_\varepsilon$. In our framework, this contribution is absorbed into the learned diffusion component. 
In practice, any jumps too small to be detected will be modelled as effective diffusion, as discussed.
% then, we may isolate and filter out jump-dominated events by thresholding increments:
% \begin{equation}
%     \bigl\{\hat{x}_n^{(j)}\bigr\}_{n,j}
%     := \bigl\{\, x_n^{(m)} \;:\; \|x_{n+1}^{(m)} - x_n^{(m)}\| > \tau \,\bigr\}
% \end{equation}
% for some cutoff $\tau\geq\epsilon$ and where $j$ ranges over the subset of particles that exceed the threshold at time $n$, and we suppress the dependence of this index set on $n$ for brevity. We use Gaussian heuristics for the choice of threshold.

% These jumps themselves may be clustered into gaussians:

\subsection{Advection-diffusion Learning} 
% \subsubsection{Advection learning}
In the absence of jumps, the lifted generator
$\widehat{\A} + \widehat{\D}$ defines a deterministic
flow on $\mathcal{G}_{\{1\}}$. From the action of the lifted generators on Gaussians (Appendix~D), the
parameters $(\mu, \Sigma)$ evolve under advection--diffusion as the ODE
\begin{equation}
    \frac{\mathrm{d}}{\mathrm{d}t}
% \begin{pmatrix} \mu \\ \Sigma \end{pmatrix}
(\pi_t,\mu_t,\Sigma_t)
= (\widehat\A + \widehat\D) \,
% \begin{pmatrix} \mu \\ \Sigma \end{pmatrix}
(\pi,\mu,\Sigma).
\end{equation}
This means we can now extract the projected generators $\L$ from our data by estimating the time derivative of the mean and covariance or the surrogate Gaussians with finite differences; we may inherit the second-order accuracy of Strang splitting by using a trapezoidal symmetrical scheme.
% differences, as opposed to first-order forward differences.
% \begin{equation}
% \label{eq:moment-ode}
% \frac{\mathrm{d}}{\mathrm{d}t}
% \begin{pmatrix} \mu \\ \Sigma \end{pmatrix}
% =
% \begin{pmatrix}
% \bigl\langle f, \mathcal{N}_{\mu, \Sigma} \bigr\rangle \\[2pt]
% \bigl\langle f(x)(x - \mu)^{\top} + (x - \mu) f(x)^{\top},\, \mathcal{N}_{\mu, \Sigma} \bigr\rangle
% + \bigl\langle g^{2},\, \mathcal{N}_{\mu, \Sigma} \bigr\rangle
% \end{pmatrix}.
% \end{equation}
% Intuitively,  Training then reduces to estimating $(f_{\theta}, g^{2}_{\vartheta})$ such that the
% flow~\eqref{eq:moment-ode} reproduces the observed cluster moments, leading to a higher-dimensional analogue of Neural ODE training. In generator terms:
  % Note that  that these identities do not rely on Euler-Maruayama-like assumptions that restrict accuracy to first order; inspired by Strang splitting, we use higher-order implicit trapezoidal differences to estimate these derivatives numerically from data.
  
  We observe that the means evolve solely due to advection. 
  % according to a smoothed vector field. 
  Strang-splitting training then consists of enforcing a "{meet-me-halfway}" condition: we advect clusters forward and backward in time, then minimize the mismatch between their means. This leads to the advection loss $L_{\text{adv}}$ for the drift $f_\theta$:
  % \footnote{Our implementation uses a slightly more intricate and robust implementation, detailed in the Appendix.}.
% \begin{equation}
%     \mmu{c}{n} + \frac{1}{2}\Tilde{f}_\theta_n^{(c)} \approx \mu^{(c)}_{n+1} - \frac{1}{2}\Tilde{f}_\theta_{n+1}^{(c)} \, , \; \text{ where } \, \Tilde{f}_\theta_n^{(c)} = \int  f_\theta\varphi_n^{(c)} \, \d x.
% \end{equation}
\begin{equation}
    L_{\text{adv}}(\theta) = \sum_{c,n}\left|\left|\mu^{(c)}_{n+1} - \mu^{(c)}_{n} - \frac{s}{2}\left( \langle f_\theta, \Tilde G^{(c)}_{n+1} \rangle + \langle f_\theta, \Tilde G^{(c)}_n \rangle \right)\right|\right|_2^2.
\end{equation}
% \begin{equation}
%     {L}_\text{adv}(\theta) = \sum_{m,n} \Bigg|\Bigg| x_{n+1}^m  - \left(x_n^m + f_\theta(x_n^m)\right) \Bigg|\Bigg|_2^2
% \end{equation}
% \subsection{Diffusion Learning}
% note that the advection network $f_\theta$ is being spatially smoothed inside the loss, so as to encourage it to learn the \textit{unsmoothed} vector field.
% \subsubsection{Diffusion learning}
% We can then train our advection component by training the drift $f_\theta$.

An analogous calculation may be carried out for diffusion. 
% Having filtered out jumps, the dynamics of the covariance tensor depend on the drift and noise in a predictable manner, 
% with an advection-induced deformation of the cluster in (in $f_\theta$) and a diffusion-induced spreading (in $gg^\top$). 
Crucially, having already obtained an estimate $f_\theta$ from the first-moment dynamics, we may subtract the drift contribution to isolate the contribution of $gg^\top_\vartheta$.
Again, we approximate these quantities empirically from clusters and match trapezoidal differences: 
\begin{equation}
{L}_\text{diff}(\vartheta)
=
\sum_{c,n}
\Big\|
\Tilde \Sigma_{n+1}^{(c)} - \Tilde\Sigma_n^{(c)} - \frac{s}{2}\left(\Delta^\A \Tilde \Sigma_n^{(c)} + \Delta^\A \Tilde\Sigma_{n+1}^{(c)} + \Delta^\D \Tilde\Sigma_n^{(c)} + \Delta^\D \Tilde\Sigma_{n+1}^{(c)}\right)
\Big\|_F^2,
\end{equation}
where
\begin{multline}
    \Delta^\A \Tilde\Sigma_n^{(c)} = \left\langle
(x - \mu_n^{(c)})f_\theta(x)^\top + f_\theta(x)(x - \mu_n^{(c)})^\top, \Tilde G_n^{(c)} \right \rangle, \quad
\Delta^\D \Tilde\Sigma_n^{(c)} = \left\langle gg^\top_\vartheta(x), \Tilde G_n^{(c)}(x) \right \rangle.
\end{multline} 
Note that both training objectives are parallelizable over both particles and time-steps; by pre-computing clusters, the cost of evaluating this loss becomes not substantially larger than that of computing a corresponding EM loss.
Note also that this procedure allows for general diffusion terms with off-diagonal and cross-effect terms, unlike EM-based methods. In particular, because no maximum-likelihood estimation is necessary, learning coupled and non-diagonal covariances becomes simple.

\subsection{Jump learning}

We learn the finite-time jump law $H^s$ directly from thresholded jump increments. For sufficiently small $s$, retained increments are assumed jump-dominated, since diffusion and drift contributions scale respectively as $O(\sqrt{s})$ and $O(s)$. We define the time-$s$ increment for the projected KFE:
\begin{equation}
\left(\pi_{t+s}^k,\mu_{t+s}^k,\Sigma_{t+s}^k\right)_k = e^{\widehat \J s} (\pi_{t},\mu_{t},\Sigma_{t}).
\end{equation}
We may then fit these increments with Gaussian mixtures to learn the projected semigroup. Note that approximating the entire increment as a jump is only accurate to leading order.

% We cluster particles and parameterize the finite-time jump law $H^s(\cdot|\mu)$ as a Gaussian mixture over $K_J$ components, with mixture parameters output by a network conditioned on the pre-jump mean $\mu$. The resulting distribution models the jump law with which a localized Gaussian centered at $\mu$ must be convolved to obtain its post-jump image.
We start by defining the jump targets $y$, which consist of a particle's future position if it was flagged as a jump, or its current position if otherwise. By doing this, we include both jump and not jump increments in order to model the full semigroup, implicitly modelling also the jump probability:
\begin{equation}
    y_n^{(m)} = x_n^{(m)} + \chi_n^{(m)}(x_{n+1}^{(m)} - x_n^{(m)})
\end{equation}
We then cluster particles into localized Gaussian packets $\widetilde{G}_n^{(c)}$. For each cluster, we define the associated target cloud
\begin{equation}
    \Tilde Y_n^{(c)}
    :=
    \left\{
        y_n^{(m)}
        :
        x_n^{(m)}
        \in
        \widetilde{G}_n^{(c)}
    \right\}.
\end{equation}
This target cloud is interpreted as an empirical sample from the action of the jump semigroup on the source cluster:
\begin{equation}
    \tilde Y_n^{(c)}
    \sim
    e^{\mathcal J s}
    \widetilde{G}_n^{(c)}
    \approx
    H^s(\cdot|\mu_n^{(c)})
    *
    G_{\widetilde{\mu}_n^{(c)},\widetilde{\Sigma}_n^{(c)}}
\end{equation}
\begin{equation}
     H^s(\cdot|\mu_n^{(c)})
    * G_{\widetilde{\mu}_n^{(c)},\widetilde{\Sigma}_n^{(c)}} = 
    \sum_{j=1}^{K_J}
    \alpha_j(\mu_n^{(c)})\,
    G_{\mu_n^{(c)}+\beta_j(\mu_n^{(c)}),\,\Sigma_n^{(c)}+\gamma_j(\mu_n^{(c)})}.
\end{equation}
Since Gaussian convolution adds means and covariances, recovering the jump law reduces to local deconvolution. We parameterize $H^s_\Theta:R^d\to(R^{d+d\times d})^{K_J}$ with a Mixture Density Network (MDN) \cite{bishop1994mdn}, which outputs the weights, means and covariances of a conditional Gaussian mixture:
\begin{equation}
    H^s_\Theta(\cdot|x)
    =
    \sum_{j=1}^{K_J}
    \alpha_j(x)
    \,
    G_{\beta_j(x),\Gamma_j(x)}.
\end{equation}
The network is trained by maximizing the conditional log-likelihood of the target clouds at their respective cluster centers:
\begin{equation}
    L_{\mathrm{jump}}(\Theta)
    =
    -
    \sum_{n,c}
    \sum_{y \in Y_n^{(c)}}
    \log
    \left[
        H^s_\Theta(\cdot|\mu_n^{(c)})
        *
        G_{\widetilde{\mu}_n^{(c)},\widetilde{\Sigma}_n^{(c)}}
    \right](y).
\end{equation}
% \begin{equation}
%     \mathcal L_{\mathrm{jump}}(\Theta)
%     =
%     -
%     \sum_{n,m}
%     \sum_{m\in \widetilde G^{(c)}_n}
%     \log
%     H^s_\Theta
%     \left(
%         y^{(m,c)}_n
%         \mid
%         \mu^{(c)}_n
%     \right)
%     *
%     G_{\mu,\Sigma}.
% \end{equation}
The convolution with the source cluster can be evaluated analytically by addition of means and covariances. Unlike classical expectation-maximization methods for Gaussian-mixture fitting, this objective is fully differentiable and naturally accommodates state-dependent conditioning within end-to-end gradient-based training.

% Note also that this strategy avoids the EM overfitting pathology. The EM loss
% $\|x^{(m)}_{n+1} - x^{(m)}_n - h\, f_{\theta}(x^{(m)}_n)\|^{2}$ admits
% the trivial solution $f_{\theta}(x^{(m)}_n) = (x^{(m)}_{n+1} - x^{(m)}_n) / h$, which fits every transition exactly and collapses the residual diffusion estimate to zero (Appendix~C.5). The Gaussian-integrated loss~\eqref{eq:adv-loss} admits no such collapse: any candidate $f_{\theta}$ must match the cluster-mean dynamics in the smoothed sense, a macroscopic constraint that cannot be satisfied by trajectory-level memorisation.

\begin{figure*}[h!]\label{fig:diagra-jump-sketch}
  \centering
    \includegraphics[width=1.08\linewidth]{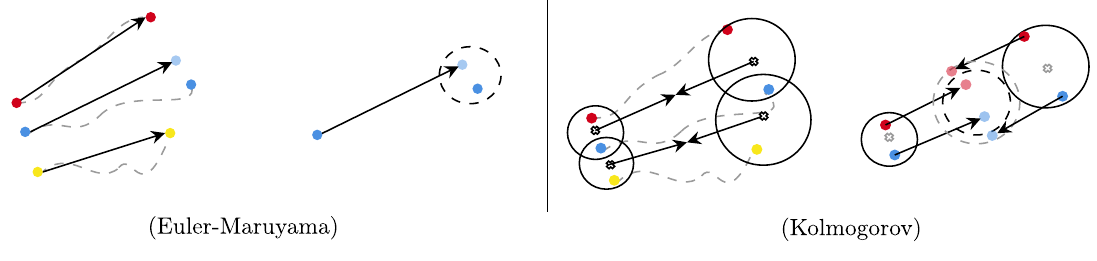}
\label{fig:learning_advection_diffusion}
    \caption{Advection-diffusion learning. Euler-Maruyama methods attempt to directly match the movement of particles between observations as drift; any remaining residue is modelled as noise. In contrast, Kolmogorov methods model the movement of particle clusters, modelling the difference in their means as advection and the remaining difference in their covariances as diffusion.}
\end{figure*}

\begin{figure*}[h!]\label{fig:diagra-jump-sketch}
  \centering
    \includegraphics[width=1.08\linewidth]{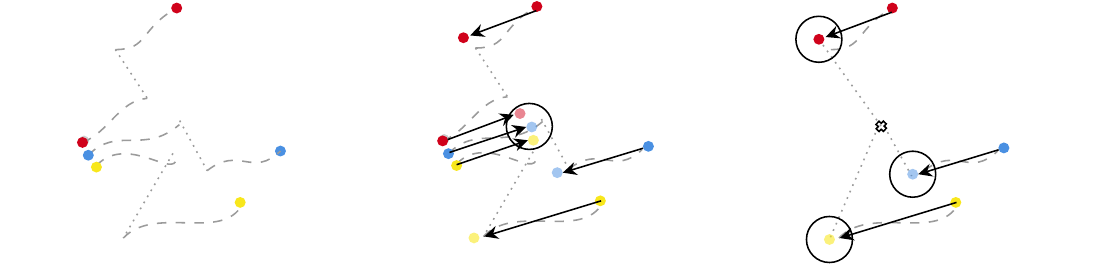}
    \label{fig:learning_jump}
    \caption{Jump learning. NKEs cluster jump origins, then model their destinations with a Gaussian mixture. The result should compose the jump distribution with the cluster's own covariance.}
\end{figure*}

\newpage 

\section{Experiments}

To validate NKEs as a conceptual framework for SDE learning, we evaluate them on a set of small and large-scale SDE-learning benchmarks in the field of scientific machine learning. For the small synthetic benchmarks, we combine three types of noise -- Geometric Brownian motion, coupled Brownian motion, and jump diffusion -- with three canonical multidimensional ODE systems -- the chaotic Lorenz oscillator, the Black-Scholes model, and the multidimensional double-well model. For the large-scale experiments, we tackle two coarse-graining problems from molecular dynamics, in which a heavy tracked particle moves in a 3D particle bath.
% , leading to unresolved sotchastic behavior on the observed particle.

Each model is trained on a set of 1024 trajectory observations. After training, 256 trajectories are generated then evaluated against a test set of 256 additional observations via three metrics: the mean squared error in the mean vector at each time, mean squared error in the covariance matrix at each time, and mean (dimensionwise) Wasserstein distance at each time. We compare NKEs to three baselines:
\begin{itemize}
    \item Parallelizable \textit{Euler-Maruyama} models, based on the classical EM scheme as implemented in \cite{shen2025neural}; these are equipped with a drift network and a diagonal noise network;
    \item Parallelizable  \textit{Trajectory Flow Matching} models, based on the implementation in \cite{zhang2024trajectory}; they are equipped with a drift network and a \textit{scalar} uncertainty network, used for the noise;
    % \item Autoregressive \textit{Neural SDE} models, as present in [?], also equipped with 
    \item Autoregressive \textit{SDE-GAN} models, inspired by the implementation in \cite{kidger2021neural}; these are equipped with a generator Neural SDE a Neural CDE discriminator.
\end{itemize}
Training is performed using the AdamW optimizer \cite{loshchilov2017decoupled} for all baselines. Each experiment was run on five seeds, and the mean and standard deviation of each
metric was reported. 
The appendix contains detailed specifications for each experiment. Due to the presence of jumps in some of the benchmarks, the baseline models had to be adapted from their original contexts; we provide implementation-specific details in the Appendix.

\section{Analysis and limitations}\label{sec:Limitations}

The results show that NKEs achieve consistently strong performance across the experiments, displaying generally superior accuracy when compared to the baseline methods, while maintaining competitive training times. In particular, NKEs deal with general forms of driving noise, including geometric and coupled Brownian motions, as well as diverse jump processes, and do not require {parametric assumptions on its form}.
This provides a unified treatment of Gaussian jumps and Poisson
processes within the same
generative framework.
NKEs also provide a more accurate framework for simulation-free Neural SDEs and EM methods: Instead of tracking noisy individual trajectories, NKEs model entire clusters of neighboring particles as probability masses, whose approximately deterministic dynamics may be learned more easily. 
% In this sense, our approach may be seen as a local analogue of moment-matching methods, which learn the variation in moments of the entire distribution.

While NKEs offer compelling advantages when learning drift and noise from data, it has limitations. First is the requirement for constant sampling rates in the presence of jumps: NKEs' noise learning step learns a fixed iteration of the semigroup $e^{\J s}$, restricting the method to fixed sampling rates. Extending jump-learning to general non-uniform sampling would require a tractable representation of the jump \textit{generator} $\J$, which seems out of reach for our Lagrangian-GMM representations, but could possible be achieve in other bases. Still, in the absence of jumps, NKEs may be trivially extended to non-constant sampling rates.

The second limitation lies in the Gaussian mixture modelling of distributions. The dimensional limitations of GMMs make Kolmogorov models as implemented here not adequate for systems in extremely high dimensions, as is the case with generative models. Likewise, while GMMs intuitively extend the range of operators we can model in the projected manifold, in practice the projection operator for $K>1$ is often ill-posed, leading to non-uniqueness in representation. When coupled with state-dependent distributions, we observe that this leads to discontinuities and poor convergence when learning complex jump patterns. Overcoming this issue would require regularization of the GMM models across the state space. 
% Still, for either non-space dependent complex jumps or space-dependent simple jumps, the current implementation of NKEs expands on the range of Lévy processes approachable through neural means.

% \newpage

% \newpage 

\begin{table}[h!]
\centering
\small
\caption{Lorenz oscillator with diagonal geometric Brownian noise.}
\begin{tabular}{lcccc}
\toprule
Method & Mean MSE $\downarrow$ & Covariance MSE $\downarrow$ & Wasserstein $\downarrow$ & TT $\downarrow$ \\
\midrule
TFM        & $1.8111 \pm 0.0401$  & $0.7516 \pm 0.0580$      & $0.3294 \pm 0.0454$ & $39.1 \pm 1.3$\\
SDE-GAN       & $2.6020 \pm 0.2698$ & $15.7637 \pm 0.0787$ & $0.511 \pm 0.0049$ & $2584 \pm 65$  \\
Euler-Maruyama     & $1.8027 \pm 0.0113$  & $0.9494 \pm 0.0451$      & $0.3065 \pm 0.0066$ & $\mathbf{33.9 \pm 0.9}$ \\
Kolmogorov (ours)  & $\mathbf{0.0471 \pm 0.0016}$ 
                & $\mathbf{0.0755 \pm 0.0089}$ 
                & $\mathbf{0.0391 \pm 0.0032}$ & $42.4 \pm 0.7$ \\
\bottomrule
\end{tabular}
\end{table}

\vspace{-0.5cm}

\begin{table}[h!]
\centering
\small
\caption{Multivariate Black-Scholes with coupled Geometric Brownian Motion.}
\begin{tabular}{lcccc}
\toprule
Method & Mean MSE $\downarrow$ & Covariance MSE $\downarrow$ & Wasserstein $\downarrow$ & TT $\downarrow$ \\
\midrule
TFM        & $\mathbf{0.0824 \pm 0.0123}$  & $1.8421 \pm 0.0612$      & $0.2145 \pm 0.0187$ & $41.3 \pm 1.5$\\
SDE-GAN         & $0.4873 \pm 0.0211$ & $1.9923 \pm 0.1021$ & $0.412 \pm 0.0038$ & $2490 \pm 70$  \\
Euler-Maruyama     & $0.1951 \pm 0.0108$  & $0.9732 \pm 0.0489$      & $0.2261 \pm 0.0105$ & $\mathbf{35.7 \pm 1.1}$ \\
Kolmogorov (ours)  & $\mathbf{0.0685 \pm 0.0104}$ & $\mathbf{0.0523 \pm 0.0061}$ & $\mathbf{0.0997 \pm 0.025}$ & $44.1 \pm 0.9$ \\
\bottomrule
\end{tabular}
\end{table}

\vspace{-0.5cm}

% \begin{table}[h!]
% \centering
% \small
% \caption{Multivariate Ornstein-Uhlenbeck with space-homogeneous, multimodal jumps.}
% \begin{tabular}{lcccc}
% \toprule
% Method & Mean MSE $\downarrow$ & Covariance MSE $\downarrow$ & Wasserstein $\downarrow$ & TT $\downarrow$ \\
% \midrule
% TFM        & $2.1043 \pm 0.0521$  & $1.0214 \pm 0.0723$      & $0.4127 \pm 0.0382$ & $40.5 \pm 1.4$\\
% Neural SDE         & $0.1985 \pm 0.0082$ & $18.5431 \pm 0.0915$ & $0.0634 \pm 0.0057$ & $2612 \pm 80$  \\
% Euler-Maruyama     & $2.0816 \pm 0.0149$  & $1.1532 \pm 0.0526$      & $0.3891 \pm 0.0083$ & $\mathbf{34.8 \pm 1.0}$ \\
% Kolmogorov (ours)  & $\mathbf{0.0612 \pm 0.0023}$ 
%                 & $\mathbf{0.0937 \pm 0.0102}$ 
%                 & $\mathbf{0.0485 \pm 0.0041}$ & $43.2 \pm 0.8$ \\
% \bottomrule
% \end{tabular}
% \end{table}

% \vspace{-0.8cm}
\begin{table}[h!]
\centering
\small
\caption{Multivariate double-well system with constant diffusion and Gaussian distributed jumps.}
\label{tab:double-well}
\begin{tabular}{lcccc}
\toprule
Method            & Mean MSE $\downarrow$ & Covariance MSE $\downarrow$ & Wasserstein $\downarrow$ & TT (s) $\downarrow$ \\
\midrule
TFM               & $\mathbf{0.0063 \pm 0.0034}$          & $0.0276 \pm 0.0036$          & $0.1809 \pm 0.0083$         & $1525\pm 63$          \\
SDE-GAN              & $0.8334 \pm 0.8334$          & $0.6720 \pm 0.0360$          & $0.4091 \pm 0.1244$         & --          \\
Euler-Maruyama    & $0.0113 \pm 0.0057$          & ${0.0160 \pm 0.0021}$ & $0.1907 \pm 0.0155$         & $\mathbf{1152 \pm 47}$ \\
Kolmogorov (ours) & $\mathbf{0.0062 \pm 0.0021}$ & $\mathbf{0.0085 \pm 0.0057}$           & $\mathbf{0.109 \pm 0.0084}$ & $1545 \pm 11$          \\

\bottomrule
\end{tabular}
\end{table}

\vspace{-0.5cm}

\begin{table}[h!]
\centering
\small
\caption{Rarefied coarse-graining benchmark. `--' indicates no meaningful convergence.}
\label{tab:ensemble_metrics}
\begin{tabular}{lcccc}
\toprule
Method            & Mean MSE $\downarrow$ & Covariance MSE $\downarrow$ & Wasserstein $\downarrow$ & TT (s) $\downarrow$ \\
\midrule
TFM                 & $1.0277 \pm 0.0916$                           &      $1.1150 \pm 2.6132$             & $0.5121 \pm 0.0024$    & $270.5 \pm 4.7$          \\
SDE-GAN & -- & -- & -- & -- \\ 
Euler-Maruyama    & $0.0423 \pm 0.0036$                           &      $1.1155 \pm 0.2612$             & $0.3121 \pm 0.0024$                         & $\mathbf{206.3 \pm 3.4}$ \\
Kolmogorov (ours) & $\mathbf{0.0343 \pm 0.0132}$                  & $\mathbf{0.1365 \pm 0.0021}$                            & $\mathbf{0.2321 \pm 0.0154}$                & $323.1 \pm 25.3$         \\

\bottomrule
\end{tabular}
\end{table}

\newpage 
\begin{figure}[h!]
    \centering
    \begin{subfigure}[t]{0.32\textwidth}
        \centering
        \includegraphics[width=0.98\linewidth]{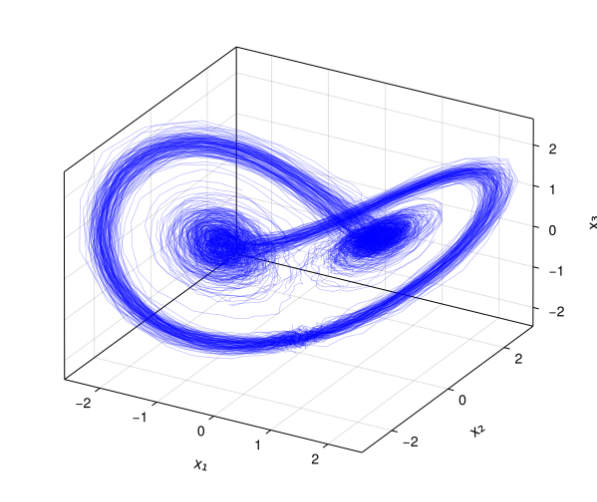}
        \caption{Ground truth.}
    \end{subfigure}
    \hfill
    \begin{subfigure}[t]{0.32\textwidth}
        \centering
        \includegraphics[width=0.98\linewidth]{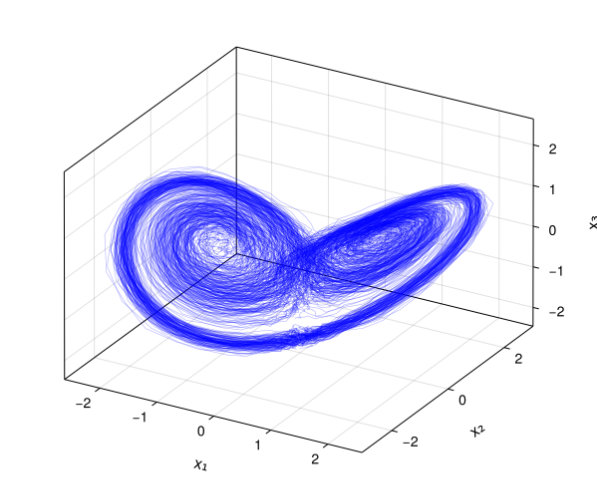}
        \caption{Euler-Maruyama.}
    \end{subfigure}
    \begin{subfigure}[t]{0.33\textwidth}
        \centering
        \includegraphics[width=0.98\linewidth]{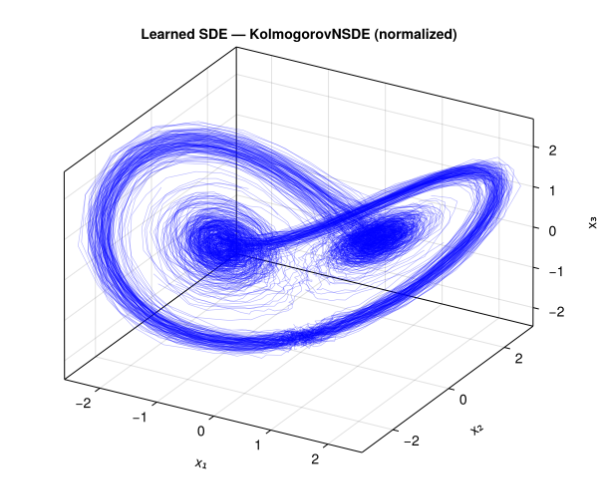}
        \caption{Kolmogorov.}
    \end{subfigure}
    \caption{Results for the Lorenz system.}
    \label{fig:lorenz}
\end{figure}
% \vspace{-0.5cm}

\begin{figure}[h!]
    \centering
    \begin{subfigure}[t]{0.32\textwidth}
        \centering
        \includegraphics[width=0.98\linewidth]{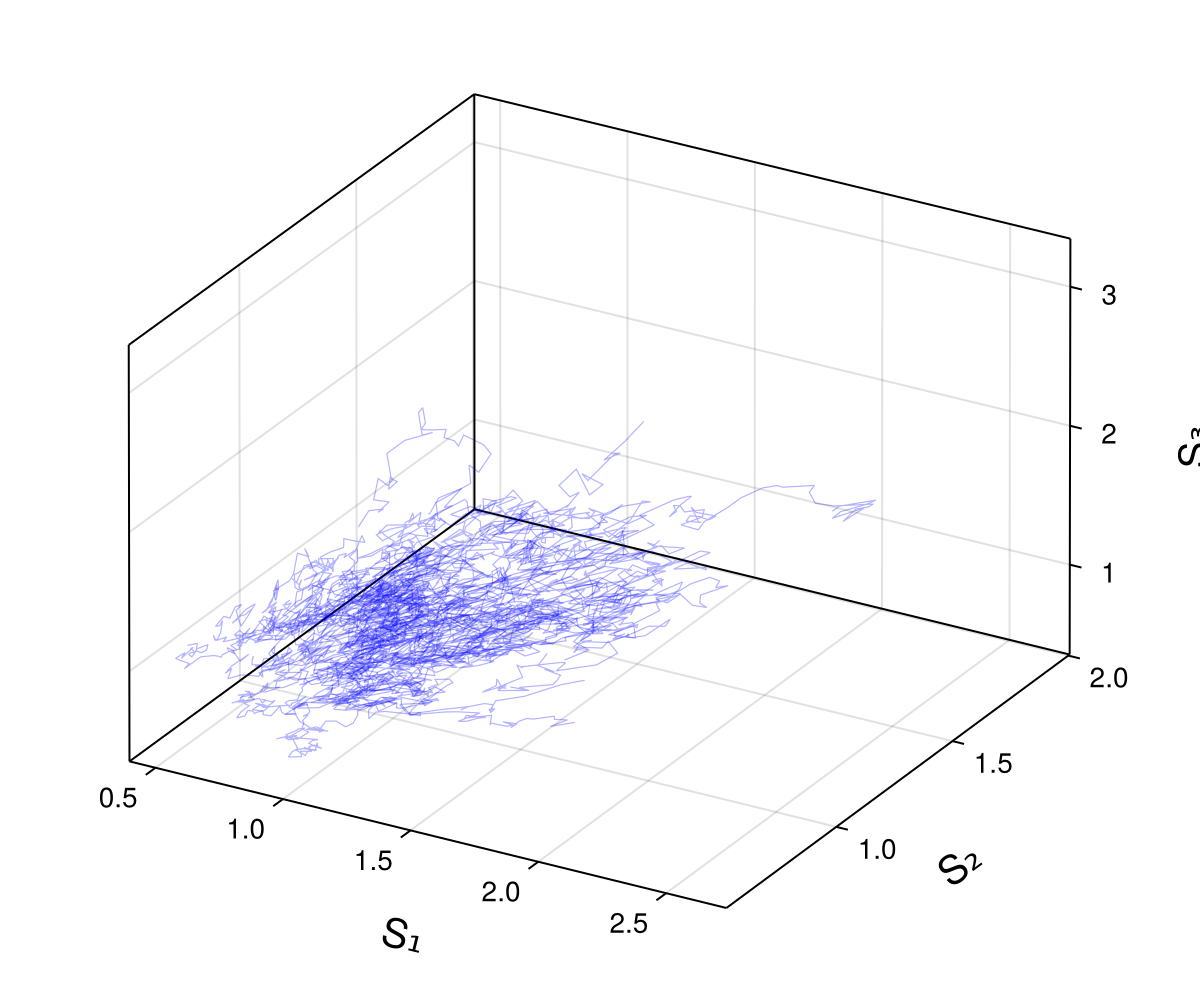}
        \caption{Ground truth.}
    \end{subfigure}
    \hfill
    \begin{subfigure}[t]{0.32\textwidth}
        \centering
        \includegraphics[width=0.98\linewidth]{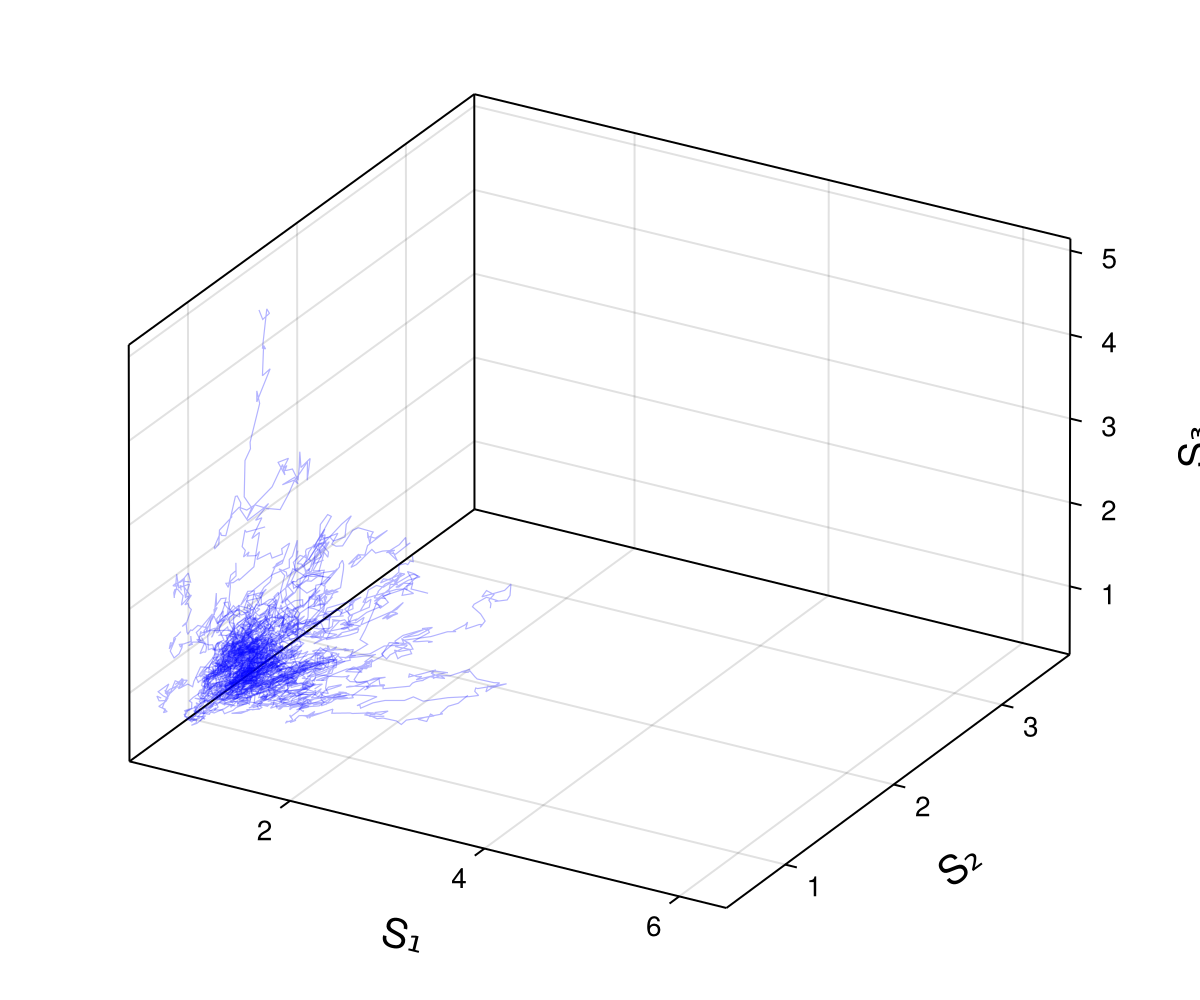}
        \caption{Euler-Maruyama.}
    \end{subfigure}
    \begin{subfigure}[t]{0.33\textwidth}
        \centering
        \includegraphics[width=0.98\linewidth]{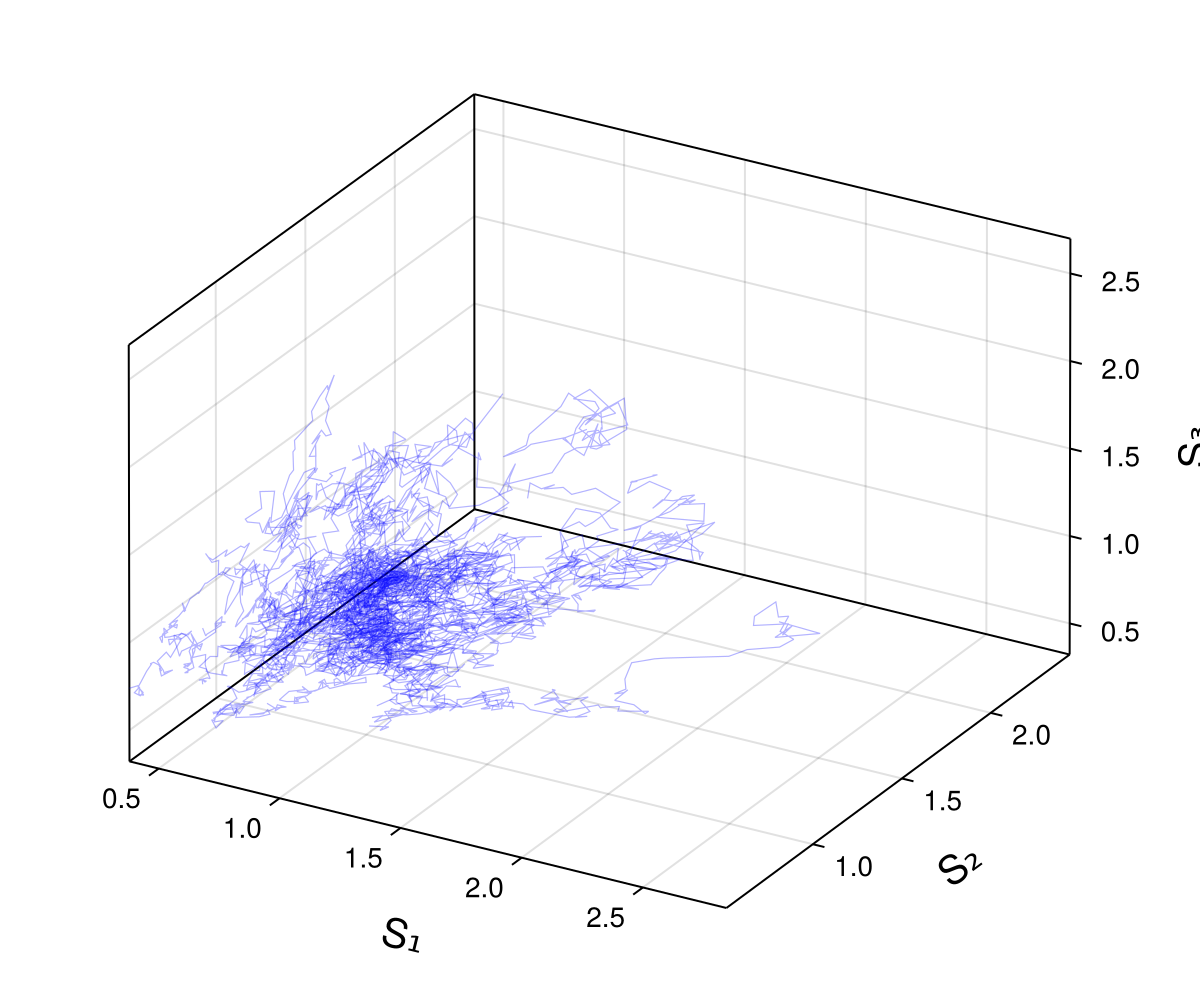}
        \caption{Kolmogorov.}
    \end{subfigure}
    \caption{Results for the Black-Scholes system.}
    \label{fig:lorenz}
\end{figure}
\begin{figure}[h!]
    \centering
    \begin{subfigure}[t]{0.32\textwidth}
        \centering
\includegraphics[width=0.98\linewidth]{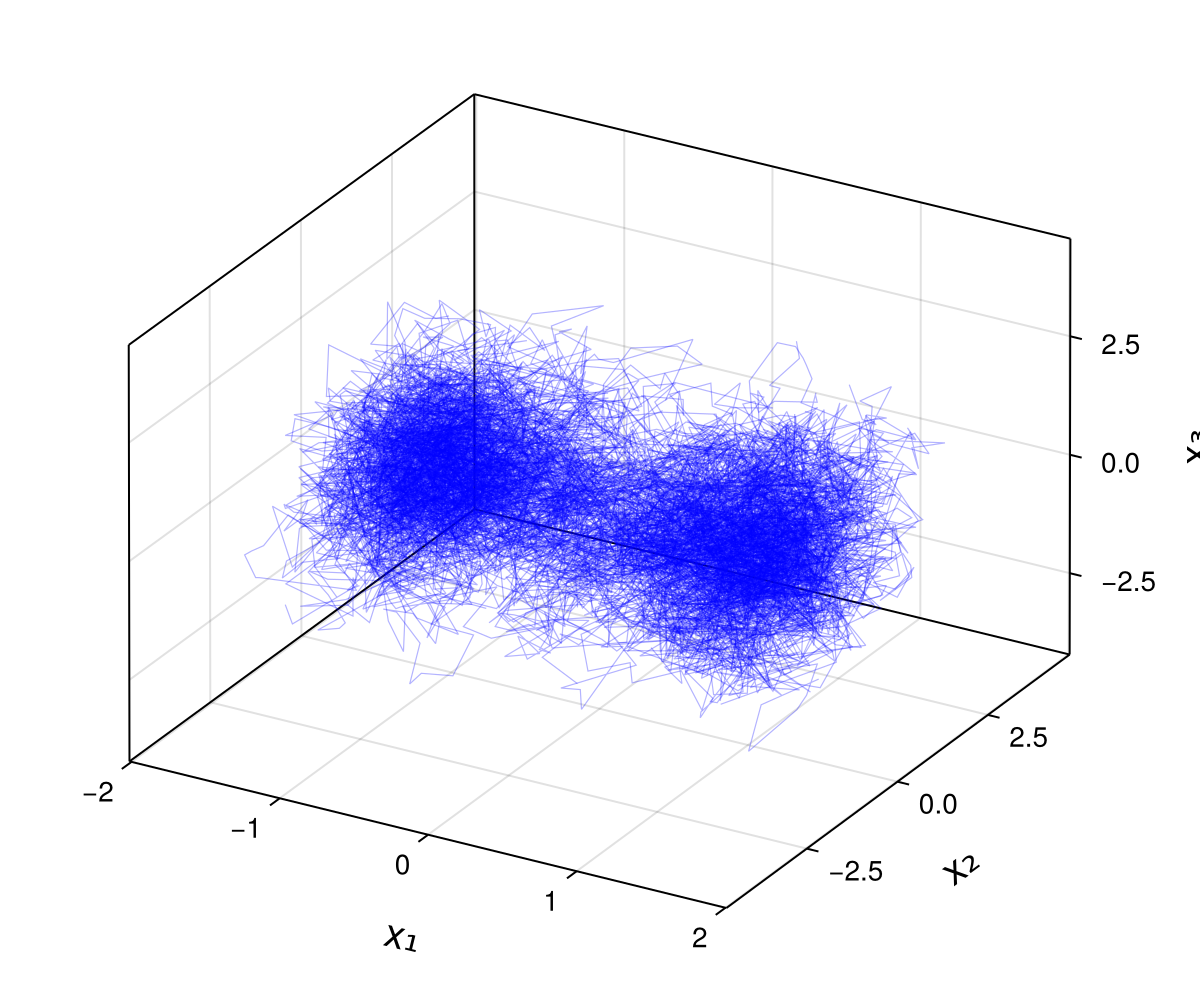}
        \caption{Ground truth.}
    \end{subfigure}
        \begin{subfigure}[t]{0.33\textwidth}
        \centering
\includegraphics[width=0.98\linewidth]{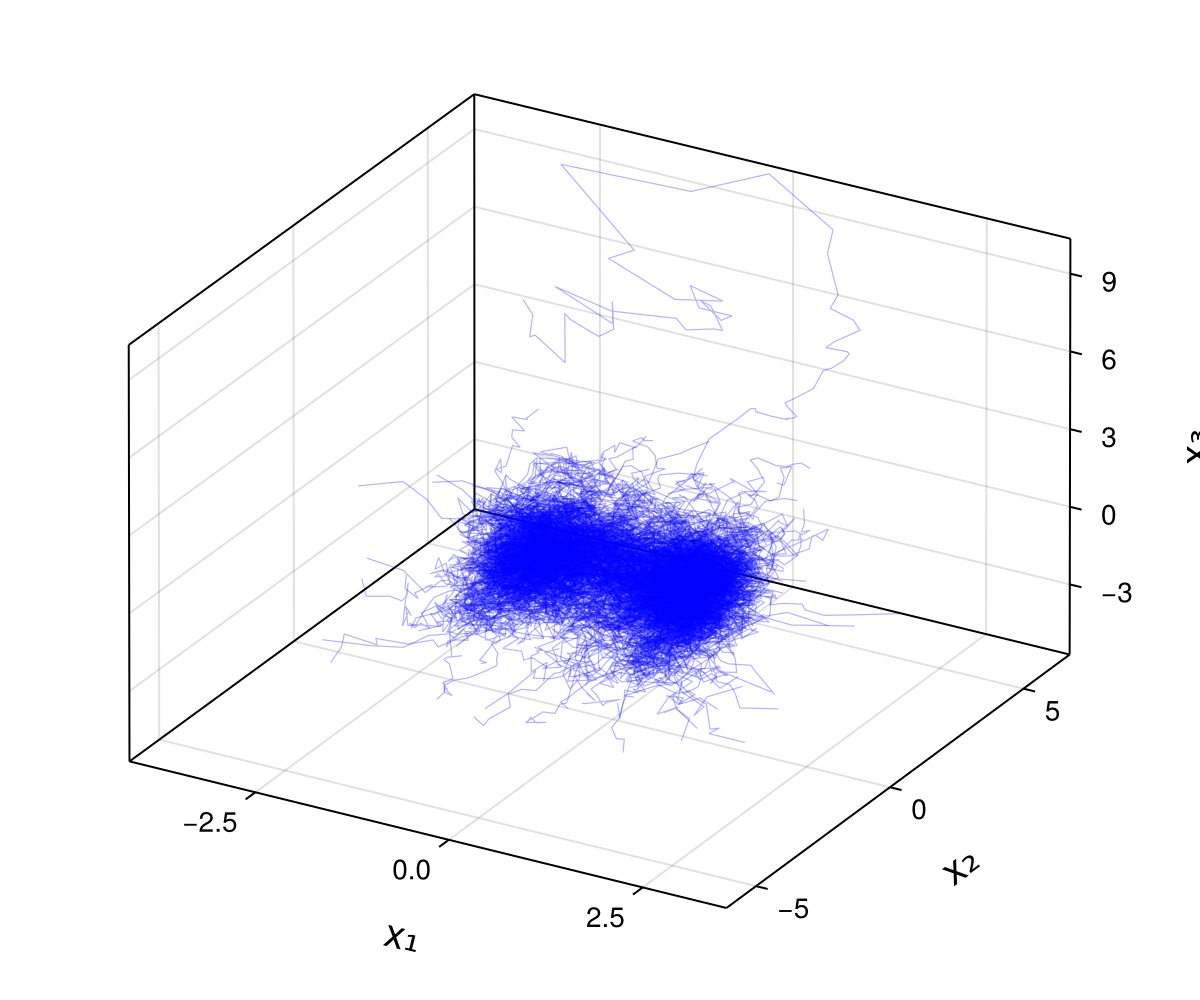}
        \caption{Kolmogorov.}
    \end{subfigure}
    \begin{subfigure}[t]{0.32\textwidth}
        \centering
        \includegraphics[width=0.98\linewidth]{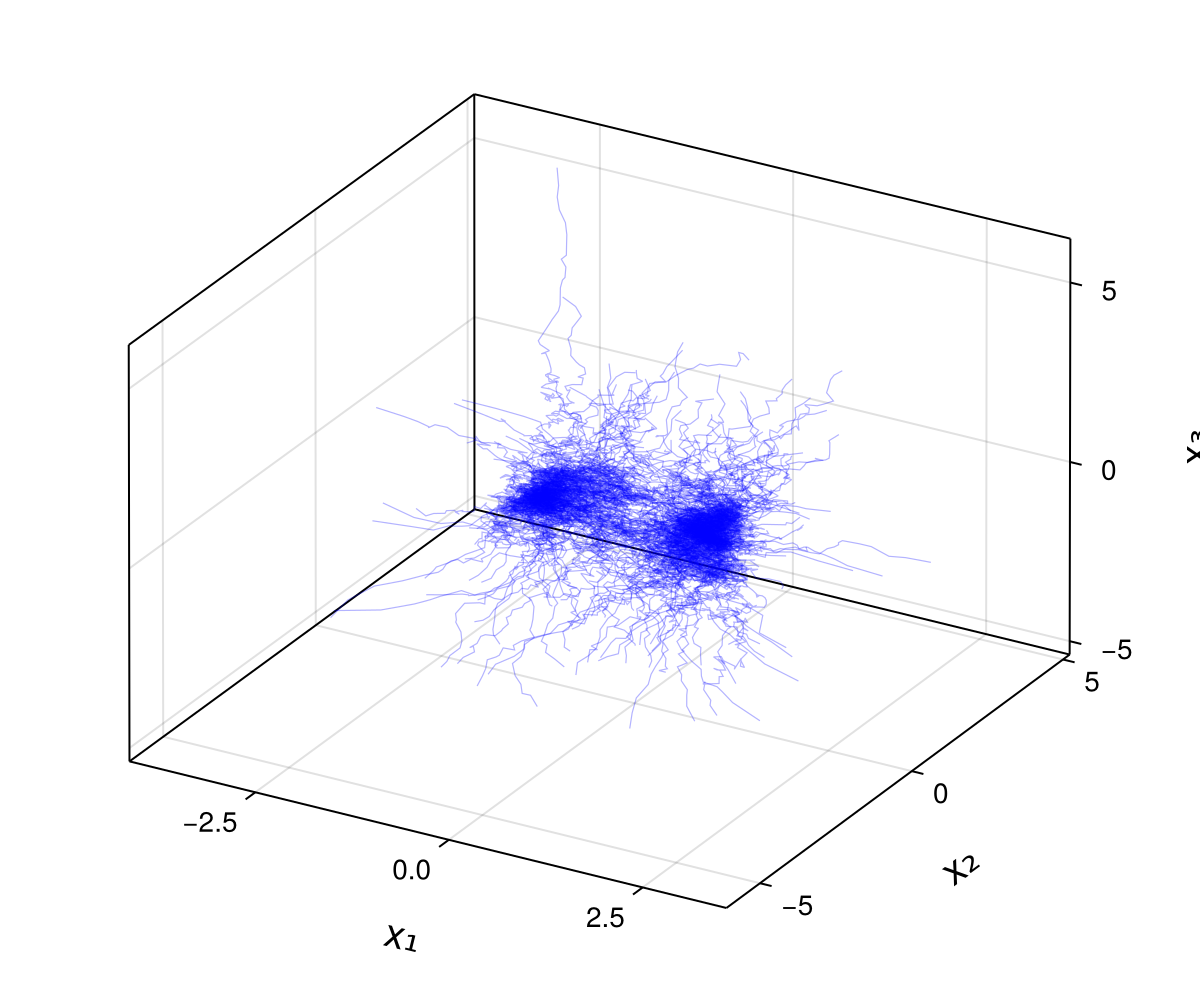}
        \caption{Euler-Maruyama.}
\end{subfigure}
    \caption{Results for the Double well system.}
    \label{fig:double-well}
\end{figure}
    
% \vspace{-0.5cm} 
\begin{figure}[h!]
    \centering
    \begin{subfigure}[t]{0.32\textwidth}
        \centering
        \includegraphics[width=0.98\linewidth]{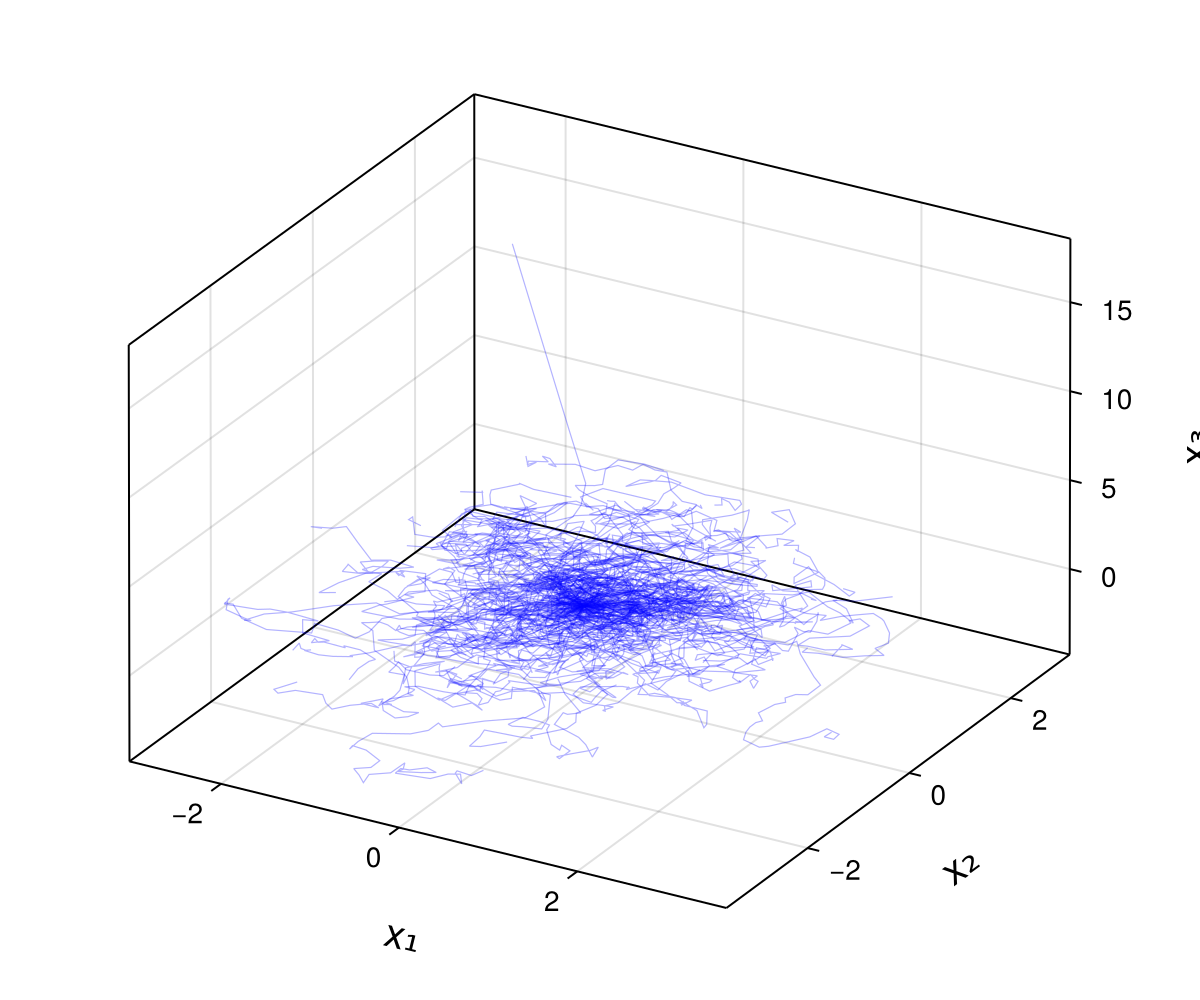}
        \caption{Ground truth.}
    \end{subfigure}
    \hfill
    \begin{subfigure}[t]{0.32\textwidth}
        \centering
        \includegraphics[width=0.98\linewidth]{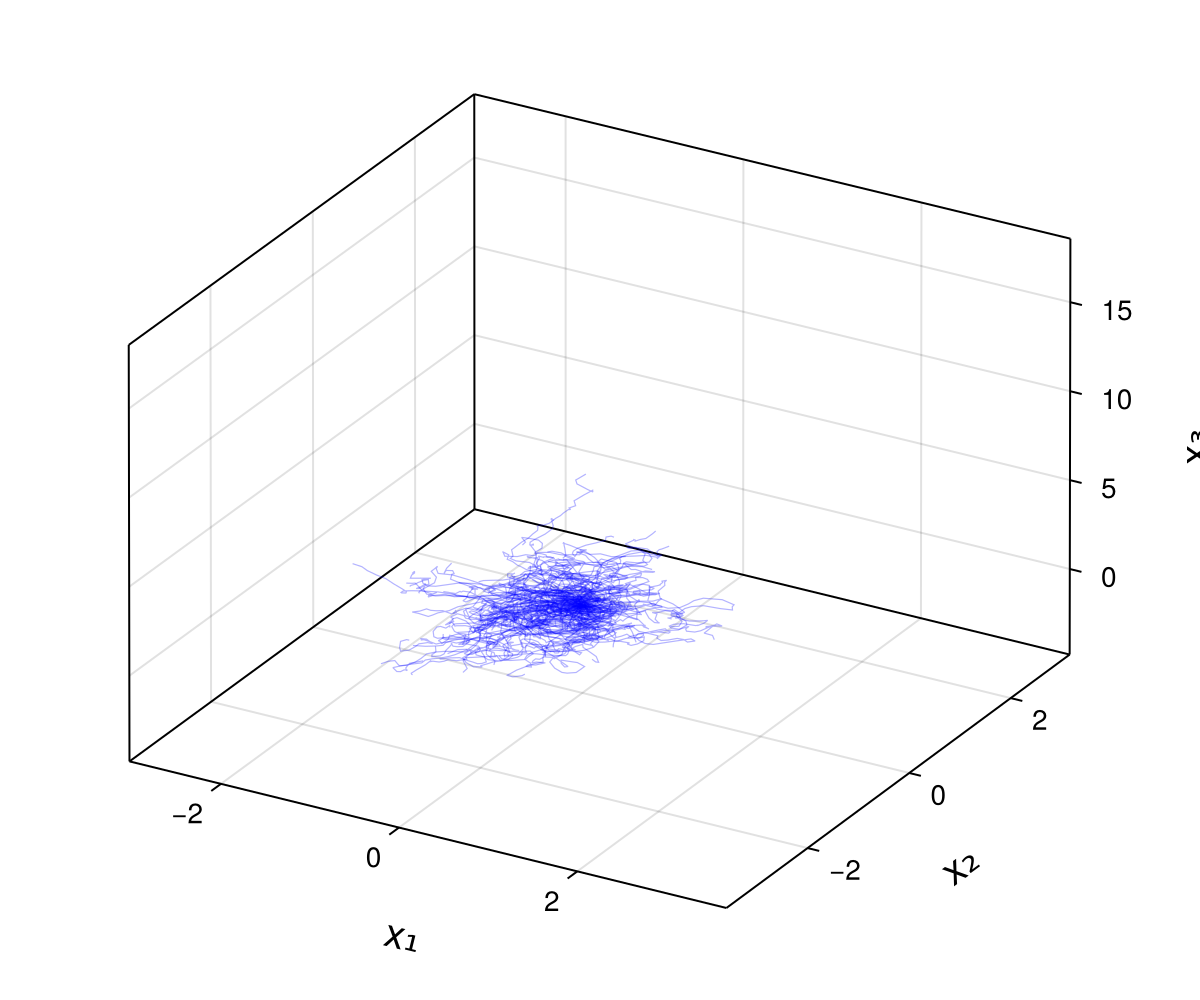}
        \caption{Euler-Maruyama.}
    \end{subfigure}
    \begin{subfigure}[t]{0.32\textwidth}
        \centering
        \includegraphics[width=0.98\linewidth]{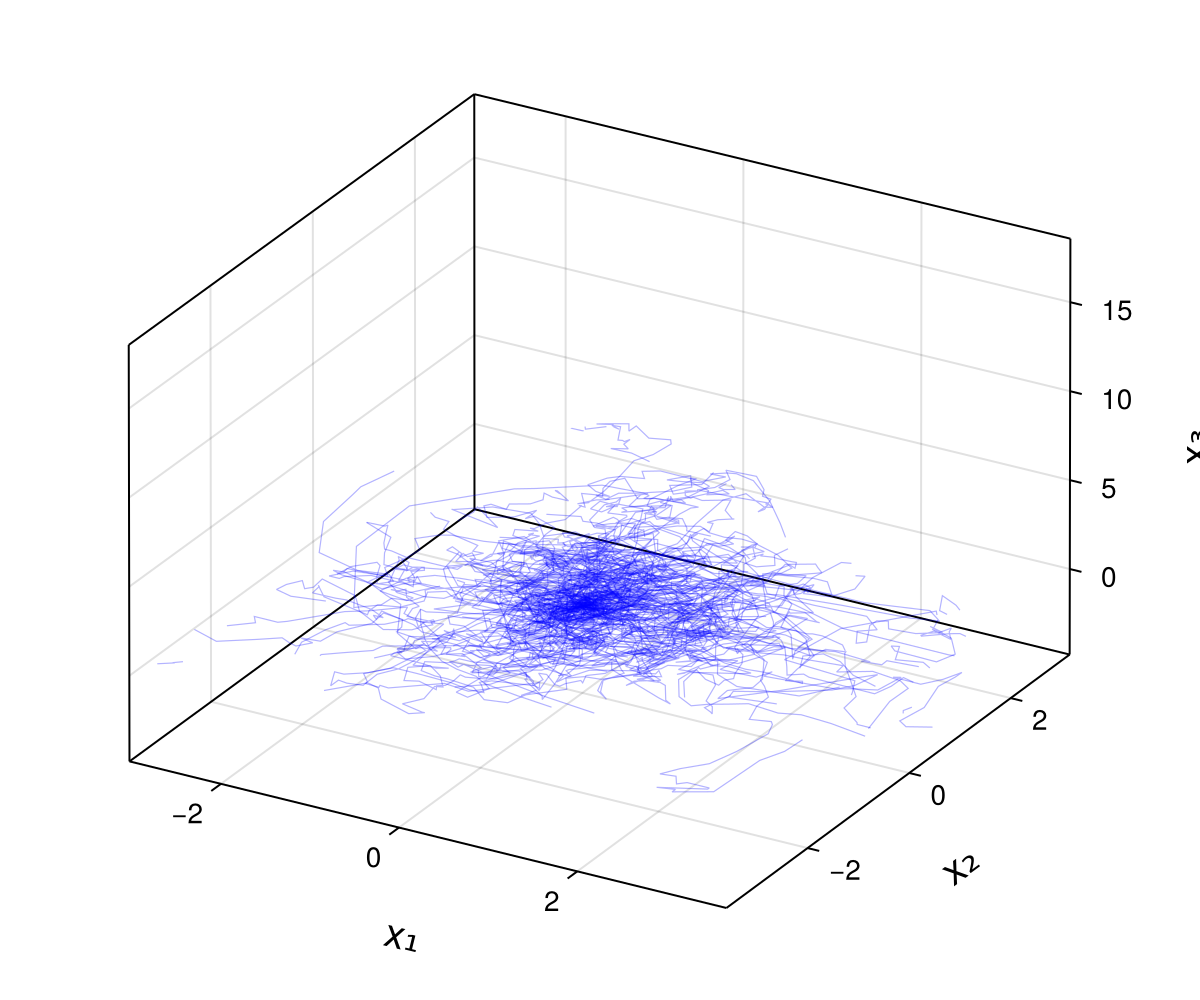}
        \caption{Kolmogorov.}
    \end{subfigure}
    \caption{Results for the rarefied coarse-graining system (positions only).}
    \label{fig:lorenz}
\end{figure}

\newpage

\bibliographystyle{unsrt}
\bibliography{bibliography}

\appendix

% ============================================================
%  NKE Paper – Full Appendix
%  Compile as a standalone supplement or \input into main.tex
% ============================================================

\appendix
\newpage 
\section{Symbols and notation}

The paper alternates between the microscopic SDE picture, given in terms of particles and their trajectories, and the macroscopic Kolmogorov picture, given in terms of densities and operators. We provie a table describing the notation and the corresponding objects across pictures.

Whenever possible, we have attempted to separate the notation as follows: 
\begin{itemize}
    \item Lower-case latin letters refer to particles and terms in the SDE picture;
    \item Calligraphic letters refer to abstract spaces and operators in the Kolmogorov picture;
    \item Greek characters refer to neural-network or Gaussian mixture parameters;
    \item Tildes refer to sample-based/empirical objects;
    \item Hats refer to projections onto the space of Gaussian Mixtures.
\end{itemize} 

\begin{table}[H]
\centering
% \small
\caption{Correspondence between the particle-level and density-level formulations.}
\begin{tabular}{p{0.42\linewidth} p{0.42\linewidth}}
\toprule
\textbf{Particle / Path Level} & \textbf{Density / Operator Level} \\
\midrule
State space $\R^d$ 
& Space of measures $\mathcal{M}(\R^d)$ \\

State trajectory $x_t$ 
& Density $p_t(x)$ \\

Discretely observed trajectory $x_n$ 
& Discretely observed density $p_{ns}(x)$ \\

Drift $f(x)$
& Advection generator $\A$ \\

Noise coefficient $g(x)$
& Diffusion generator $\D$ \\

Jump distribution  $h(x)$
& Jump generator $\J$ \\

% Drift action $ x_0 + \int_0^sf(x_\tau) \d \tau$ & 
% Advection action $e^{\A s}p_0$ \\

Particle increment $x_{n+1}-x_n$
& Infinitesimal generator action $(e^{\mathcal L s} - 1)p_{ns}$ \\

Single particle $x^{(m)}_n$
& Delta measure $\delta_{x^{(m)}_n}$ \\

Local particle cluster $\tilde{G}_n^{\{c\}}$ 
& Gaussian packet $G_{\mu,\Sigma}$ \\

Cluster mean/covariance $\tilde \mu^{\{c\}}_n,\Sigma^{\{c\}}_n$
& Gaussian parameters $\mu,\Sigma$ \\

% Euler--Maruyama step $x_{h} = x_{0} + sf(x_0) + \sqrt{s}g(x_0)\xi$
% & Lie--Trotter splitting $p_s = e^{\D s}e^{\A s}p_0$\\

% Two-sided trajectory matching $x_{s} = x_{0} + sf(x_0)/2 + + \sqrt{h}g(x_0)\xi$
% & Strang splitting $p_h = e^{\D h}e^{\A h}p_0$\\

Jump finite-time distribution $H^s$
& Jump semigroup $e^{\J s}$\\

Jump thresholding $\|x_{n+1}-x_n\|>\tau$
& L\'evy decomposition $\J_{<\varepsilon}+\J_{\ge\varepsilon}$ \\

Learned drift $f_\theta$
& Learned lifted advection generator $\widehat \A_\theta$ \\

Learned diffusion $g_\vartheta^2$
& Learned lifted diffusion generator $\widehat\D_\vartheta$ \\

Learned finite-time jump distribution $H^s_\Theta$
& Learned lifted jump semigroup $e^{\hat \J s}_\Theta$\\

\bottomrule
\end{tabular}
\end{table}

Additionally, we have the following operations:
\begin{itemize}
    \item The divergence $\nabla \cdot$ refers to the sum:
    $$
    \nabla \cdot a(x) = \sum_{i=1}^d \frac{\d}{\d x_i} a_i(x),
    $$
    where $a_i$ and $x_i$ loop over the vector positions;
    \item The operator $\nabla :$ refers to the sum:
    $$
    \nabla : a(x) = \sum_{i,j=1}^d \frac{\d^2}{\d x_i \d x_j} a_{ij}(x),
    $$
    where $a_{ij}$ loops over matrix entries;
    \item The inner product $\langle \phi, \varphi \rangle$ refers to the classical inner product in integral form:
    $$
    \langle \phi, \varphi \rangle = \int_{\R^d} \phi(x) \varphi(x) \d x;
    $$
    \item The superscript $^\star$ refers to an operator's classical adjoint $\langle \phi, \L \varphi \rangle = \langle \L^\star \phi,  \varphi \rangle$.
    
\end{itemize}

\newpage 

\section{Projection onto the Gaussian Mixture Manifold}
\label{app:projection}
% ============================================================

We study the projection operator
\[
\Pi_K :
\mathcal M(\mathbb R^d)
\to
\mathcal G_K,
\]
which maps probability densities onto the family of
$K$-component Gaussian mixtures. We characterize the geometry of this
projection, discuss uniqueness and non-uniqueness properties, and
connect them to the lifted dynamics used throughout NKEs.

% ============================================================
\subsection{The Gaussian Mixture Family}
% ============================================================

A Gaussian mixture with $K$ components takes the form
\begin{equation}
    p_\lambda(x)
    =
    \sum_{k=1}^{K}
    \pi^{(k)}
    G_{\mu^{(k)},\Sigma^{(k)}}(x),
    \qquad
    \sum_{k=1}^{K}\pi^{(k)}=1,
    \quad
    \pi^{(k)}>0,
\end{equation}
with parameters
\[
\lambda
=
\left(\pi^{k},\mu^{k},\Sigma^{k}\right)
_{k=1}^{K}.
\]
The corresponding family is
\begin{equation}
    \mathcal G_K
    :=
    \{
        p_\lambda
        :
        \lambda\in\Lambda_K
    \},
\end{equation}
where
\[
\Lambda_K
=
\Delta_{K-1}
\times
(\mathbb R^d\times\mathcal S_{++}^d)^K,
\]
with $\Delta_{K-1}$ the $K-1$ simplex and $S_{++}^d$ the space of positive definite matrices. For $K=1$, $\mathcal G_1$ forms a smooth statistical manifold under the
Fisher--Rao metric~\citep{amari2000methods}. In particular, the Fisher
information is strictly positive definite:
\begin{equation}
I(\mu,\Sigma)
=
\begin{pmatrix}
\Sigma^{-1}
&
0
\\
0
&
\tfrac12
\Sigma^{-1}\otimes\Sigma^{-1}
\end{pmatrix}.
\end{equation}

For $K>1$, the geometry becomes singular whenever:
\begin{enumerate}
    \item two Gaussian components coincide,
    \item or a mixture weight vanishes.
\end{enumerate}

Consequently, $\mathcal G_K$ becomes a stratified space rather than a
smooth manifold~\citep{ho2016singularity}.

% ============================================================
\subsection{Projection Operator}
% ============================================================

We will define $\Pi$ in terms a divergence $D$ on $\mathcal M(\mathbb R^d)$:
\begin{equation}
    \Pi_K(p)
    :=
    \arg\min_{q\in\mathcal G_K}
    D(p\|q).
\end{equation}
Choosing $D=\mathrm{KL}$ yields the population maximum-likelihood
projection:
\begin{equation}
    \Pi_K^{\mathrm{KL}}(p)
    =
    \arg\max_{q\in\mathcal G_K}
    \int
    p(x)\log q(x)\,\mathrm dx.
\end{equation}
In practice, this corresponds to Gaussian-mixture fitting through
expectation-maximization (EM)~\citep{dempster1977maximum}.

% ============================================================
\subsubsection{Uniqueness for $K=1$.}
% ============================================================

\begin{proposition}
For any density $p$ with finite second moments,
the KL projection onto $\mathcal G_1$ is uniquely given by moment
matching:
\begin{equation}
    \Pi_1^{\mathrm{KL}}(p)
    =
    G_{\mu^\ast,\Sigma^\ast},
\end{equation}
where
\begin{equation}
    \mu^\ast
    =
    \int x\,p(x)\,\mathrm dx,
    \qquad
    \Sigma^\ast
    =
    \int
    (x-\mu^\ast)(x-\mu^\ast)^\top
    p(x)\,\mathrm dx.
\end{equation}
\end{proposition}

\begin{proof}[Proof sketch]
The KL divergence
\[
\mathrm{KL}
(
p\|G_{\mu,\Sigma}
)
\]
is strictly convex in $(\mu,\Sigma)$ because the Fisher information on
$\mathcal G_1$ is positive definite. Differentiating with respect to
$\mu$ and $\Sigma$ yields the classical moment-matching conditions, thus clearly
 unique.
\end{proof}

This uniqueness is fundamental for the advection and diffusion losses,
which only rely on single-Gaussian projections.

\subsection{Non-Uniqueness for $K>1$}

For $K>1$, the projection is generally not unique at the level of mixture parameters. This non-uniqueness is intrinsic to the Gaussian-mixture representation itself and does not depend on the particular optimization algorithm used to compute the projection.

\begin{proposition}[Non-uniqueness of Gaussian-mixture projections]
Let $K>1$ and define the projection
\[
\Pi_K(p)
:=
\arg\min_{q\in\mathcal G_K}
D(p\|q),
\]
where $D$ is any divergence that depends only on the represented density $q$.

Then the projection is generally non-unique in parameter space. In particular:

\begin{enumerate}
\item \textbf{Permutation symmetry.}
If $\lambda^\star$ is an optimal set of mixture parameters, then any permutation of its Gaussian components yields the same density and is therefore also optimal.

\item \textbf{Redundant representations.}
If $p\in\mathcal G_{K^\star}$ for some $K^\star < K$, then infinitely many distinct parameter vectors in $\mathcal G_K$ represent the same density and attain the same projection error.
\end{enumerate}
\end{proposition}

\begin{proof}[Proof sketch]
Permuting Gaussian components leaves the mixture density unchanged. Therefore, if
\[
p_{\lambda^\star}
=
\sum_{k=1}^{K}
\pi^{(k)}
G_{\mu^{(k)},\Sigma^{(k)}},
\]
is optimal, then any relabeling of the components produces the same density and achieves the same objective value. For the second claim, suppose a component
\[
\pi G_{\mu,\Sigma}
\]
appears in the mixture. For any $\alpha\in(0,1)$,
\[
\pi G_{\mu,\Sigma}
=
\alpha\pi G_{\mu,\Sigma}
+
(1-\alpha)\pi G_{\mu,\Sigma}.
\]
Thus a single component may be replaced by two identical components whose weights sum to the original weight without changing the represented density. Repeating this construction produces infinitely many distinct parameter vectors corresponding to the same density.
\end{proof}

The non-uniqueness described above is structural and reflects the non-identifiability of Gaussian-mixture parameterizations. Additional optimization-related difficulties may arise because the projection objective is generally non-convex and may possess singular or degenerate stationary points. 
% These phenomena are distinct from the intrinsic non-uniqueness of the representation itself.

\subsubsection{Practical Projection in NKEs}
% ============================================================

The non-uniqueness of these projection will only be an issue when estimating state-depedent jump distributions, as in this case the non-uniqueness of the projection may lead to discontinuity. NKEs use two distinct projection mechanisms.

\paragraph{Advection--diffusion projection ($K=1$).}
Empirical clusters
$\widetilde G_n^{\{c\}}$
are projected through moment matching onto Gaussian packets
\[
G_{
\widetilde\mu_n^{\{c\}},
\widetilde\Sigma_n^{\{c\}}
}.
\]

This projection is unique and fully explicit.

\paragraph{Jump projection ($K\ge1$).}
Finite-time jump laws
$H^s(\cdot|x)$
are approximated through conditional Gaussian mixtures
parameterized by
$H_\Theta^s(\cdot|x)$. We expect the network's smoothness to act as implicit regularization. To stabilize training, we furthermore:
\begin{enumerate}
    \item use small mixture counts,
    \item regularize covariances,
    \item prune low-weight components,
    \item enforce spatial smoothness.
\end{enumerate}

\section{Lifted Dynamics}

This appendix provides theoretical justification for the approximations underlying Neural Kolmogorov Equations. We show that:
\begin{enumerate}
    \item Localized Gaussian packets remain approximately Gaussian under short-time advection--diffusion evolution;
    \item Sufficiently small jumps behave as an effective diffusion process;
    \item State-dependent jump distributions may be frozen at the packet center with controlled error.
\end{enumerate}

Throughout this section, we assume

\[
f\in C_b^3(\mathbb R^d,\mathbb R^d),
\qquad
a:=gg^\top\in C_b^2(\mathbb R^d,\mathbb R^{d\times d}),
\]

and that all jump measures possess finite third moments. We denote by

\[
\mathcal L=\mathcal A+\mathcal D+\mathcal J
\]

the Kolmogorov generator,

\[
\mathcal A p
=
-\nabla\cdot(fp),
\qquad
\mathcal D p
=
\frac12\nabla^2:(ap),
\]
% and
\[
\mathcal J p(x)
=
\int_{\mathbb R^d}
% \bigl[
p(x-y)h(x-y,y)-p(x)
% \bigr]
h(x,y) \, \d y.
\]

\subsection{Short-time Gaussianity}

We now justify the Gaussian lifting procedure used throughout the paper, in particular the claim that Gaussianity is preserved to second order in time. 
Let
\[
p_0
=
G_{\mu,\Sigma}
\]
be a Gaussian packet and define
\[
p_s
=
e^{(\mathcal A+\mathcal D)s}p_0.
\]
The projected mean and covariance increments are
\[
\Delta\mu
=
\langle f,p_0\rangle,
\]
and
\[
\Delta\Sigma
=
\Big\langle
(x-\mu)f^\top
+
f(x-\mu)^\top
+
a,
p_0
\Big\rangle.
\]

Let

\[
q_s
=
G_{\mu+s\Delta\mu,
\Sigma+s\Delta\Sigma}
\]

denote the Gaussian obtained by evolving only the first two moments.

\begin{lemma}[Growth of higher cumulants]
Let $p_t = e^{(A+D)t}G_{\mu,\Sigma}$, with
\[
Ap = -\nabla\cdot(fp),
\qquad
Dp = \frac12 \nabla^2 : (ap),
\qquad
a := gg^\top .
\]
Assume $f\in C_b^3(\mathbb R^d,\mathbb R^d)$ and
$a\in C_b^2(\mathbb R^d,\mathbb R^{d\times d})$. Let $\kappa_\alpha(t)$ denote the cumulant of $p_t$ associated with a multi-index $\alpha$ with $|\alpha|\geq 3$. Then, for sufficiently small $t$,
\[
\kappa_\alpha(t) = O(t).
\]
\end{lemma}

\begin{proof}[Proof sketch]
We work entirely in weak form. For any smooth test function $\varphi$,
\[
\frac{d}{dt}\langle \varphi,p_t\rangle
=
\langle \varphi,(\A+\D)p_t\rangle
=
\langle (\A^\star+\D^\star)\varphi,p_t\rangle ,
\]
where
\[
\A^\star \varphi = f\cdot\nabla\varphi,
\qquad
\D^\star \varphi = \frac12 a:\nabla^2\varphi .
\]
Let
\[
m_\alpha(t):=\langle x^\alpha,p_t\rangle
\]
be the raw moment associated with the multi-index $\alpha$. Taking
$\varphi(x)=x^\alpha$ gives
\[
\frac{d}{dt}m_\alpha(t)
=
\left\langle
f\cdot\nabla x^\alpha
+
\frac12 a:\nabla^2 x^\alpha,
p_t
\right\rangle .
\]
Since $x^\alpha$ is a polynomial, $\nabla x^\alpha$ and $\nabla^2 x^\alpha$ are polynomials of degrees
$|\alpha|-1$ and $|\alpha|-2$, respectively. Under the assumed boundedness of $f$ and $a$, the right-hand side is bounded on a short time interval by a constant depending only on finitely many moments of $p_t$. Since $p_0=G_{\mu,\Sigma}$ has finite moments of all orders, standard moment estimates imply that
\[
\frac{d}{dt}m_\alpha(t)=O(1),
\]
and therefore
\[
m_\alpha(t)=m_\alpha(0)+O(t).
\]

Now let $\kappa_\alpha(t)$ be the corresponding cumulant. Cumulants are polynomial functions of raw moments up to order $|\alpha|$:
\[
\kappa_\alpha(t)
=
P_\alpha\bigl(\{m_\beta(t):|\beta|\leq |\alpha|\}\bigr),
\]
for a universal polynomial $P_\alpha$. Since each raw moment satisfies
\[
m_\beta(t)=m_\beta(0)+O(t),
\]
it follows that
\[
\kappa_\alpha(t)
=
\kappa_\alpha(0)+O(t).
\]
Because the initial packet $p_0=G_{\mu,\Sigma}$ is Gaussian, all cumulants of order $|\alpha|\geq 3$ vanish:
\[
\kappa_\alpha(0)=0.
\]
Hence
\[
\kappa_\alpha(t)=O(t),
\]
as claimed.
\end{proof}

Therefore a localized Gaussian packet remains approximately Gaussian
under short-time advection--diffusion evolution, with second-order
error in relative entropy.

% See \cite{hall1992bootstrap,pavliotis2014stochastic}.

% \newpage 

\subsection{Small-jump approximation}

We now justify the approximation of sufficiently small jumps by an effective advection--diffusion process. For a state-dependent jump kernel $h(x,dy)$, the backward (adjoint) jump generator acting on test functions is:
\begin{equation}
\mathcal J^\star \phi(x)
=
\int_{\mathbb R^d}
\bigl[
\phi(x+y)-\phi(x)
\bigr]
\,h(x,dy).
\end{equation}

Equivalently, $\mathcal J$ admits the gain--loss representation

\begin{equation}
(\mathcal Jp)(x)
=
\int_{\mathbb R^d}
p(x-y)\,
h(x-y,dy)
-
p(x)
\int_{\mathbb R^d}
h(x,dy).
\end{equation}

The first term represents mass arriving at $x$ from points $x-y$, while the second represents mass leaving $x$ through jumps to other locations. Restricting attention to jumps smaller than a threshold $\varepsilon>0$, define
\begin{equation}
\mathcal J^\star_{<\varepsilon}\phi(x)
=
\int_{|y|<\varepsilon}
\bigl[
\phi(x+y)-\phi(x)
\bigr]
\,h(x,dy).
\end{equation}

\begin{proposition}[Small-jump diffusion approximation]
Suppose $\phi\in C^3(\mathbb R^d)$. Define
\begin{equation}
\beta(x)
=
\int_{|y|<\varepsilon}
y\,h(x,dy),
\end{equation}
and
\begin{equation}
\gamma(x)
=
\int_{|y|<\varepsilon}
yy^\top h(x,dy).
\end{equation}
Then

\begin{equation}
\mathcal J^\star_{<\varepsilon}\phi
=
\beta\cdot\nabla\phi
+
\frac12
\Gamma:\nabla^2\phi
+
\mathcal R^\star_\varepsilon\phi,
\end{equation}

where

\begin{equation}
|\mathcal R^\star_\varepsilon\phi(x)|
\le
C
\|\nabla^3\phi\|_\infty
\int_{|y|<\varepsilon}
|y|^3 h(x,dy).
\end{equation}
\end{proposition}

\begin{proof}
Applying Taylor's theorem,

\begin{equation}
\phi(x+y)
=
\phi(x)
+
y^\top\nabla\phi(x)
+
\frac12
y^\top\nabla^2\phi(x)y
+
R_3(x,y),
\end{equation}
with
\begin{equation}
|R_3(x,y)|
\le
C|y|^3
\|\nabla^3\phi\|_\infty.
\end{equation}
Substituting into $\mathcal J^\star_{<\varepsilon}$ and integrating term-by-term yields

\begin{equation}
\mathcal J^\star_{<\varepsilon}\phi
=
\beta\cdot\nabla\phi
+
\frac12
\gamma:\nabla^2\phi
+
\mathcal R^\star_\varepsilon\phi.
\end{equation}

The remainder estimate follows immediately.
\end{proof}

Passing to the forward operator by duality we get the usual advection-diffusion form:
\begin{equation}
\mathcal J_{<\varepsilon}p
=
-\nabla\cdot(\beta p)
+
\frac12
\nabla^2:(\gamma p)
% + \mathcal R_\varepsilon p,
\end{equation}

% where the remainder obeys the weak estimate

% \begin{equation}
% \left|
% \int
% \phi\,
% \mathcal R_\varepsilon p
% \right|
% \le
% C
% \|\nabla^3\phi\|_\infty
% \int
% p(x)
% \left[
% \int_{|y|<\varepsilon}
% |y|^3 h(x,dy)
% \right]
% dx.
% \end{equation}

Consequently, sufficiently small jumps behave, to second order in jump size, as an effective advection--diffusion process with drift coefficient $\beta$ and diffusion tensor $\Gamma$. This recovers the classical diffusion approximation underlying the Lévy--Khintchine representation and the construction of Lévy-type diffusions.

\subsection{Jump freezing} 

\begin{proposition}[Freezing the finite-time jump operator]
Fix \(s>0\). Let \(H_s(x,y)\) denote the density of the finite-time jump
increment \(y\in\mathbb{R}^d\) associated with a particle located at
\(x\in\mathbb{R}^d\). Define the state-dependent jump operator
\[
(\mathcal{H}_s p)(z)
:=
\int_{\mathbb{R}^d}
H_s(x,z-x)p(x)\,\mathrm{d}x,
\]
and the operator obtained by freezing the jump law at \(\mu\),
\[
(\mathcal{H}_s^\mu p)(z)
:=
\int_{\mathbb{R}^d}
H_s(\mu,z-x)p(x)\,\mathrm{d}x
=
\bigl(H_s(\mu,\cdot)*p\bigr)(z).
\]

Assume that \(H_s\) is Lipschitz continuous in its spatial argument in
\(L^1\), namely that there exists \(L_s>0\) such that
\[
\int_{\mathbb{R}^d}
\left|
H_s(x,y)-H_s(x',y)
\right|
\,\mathrm{d}y
\leq
L_s\|x-x'\|
\]
for every \(x,x'\in\mathbb{R}^d\). Then
\[
\left\|
\mathcal{H}_sG_{\mu,\Sigma}
-
\mathcal{H}_s^\mu G_{\mu,\Sigma}
\right\|_{L^1}
\leq
L_s\sqrt{\operatorname{tr}\Sigma}.
\]
Consequently,
\[
\mathcal{H}_sG_{\mu,\Sigma}
-
H_s(\mu,\cdot)*G_{\mu,\Sigma}
\longrightarrow 0
\qquad\text{in }L^1(\mathbb{R}^d)
\]
as \(\operatorname{tr}\Sigma\to0\).
\end{proposition}

\begin{proof}
By the definitions of \(\mathcal{H}_s\) and
\(\mathcal{H}_s^\mu\),
\[
\begin{aligned}
&
\left\|
\mathcal{H}_sG_{\mu,\Sigma}
-
\mathcal{H}_s^\mu G_{\mu,\Sigma}
\right\|_{L^1}
\\
&=
\int_{\mathbb{R}^d}
\left|
\int_{\mathbb{R}^d}
\left[
H_s(x,z-x)-H_s(\mu,z-x)
\right]
G_{\mu,\Sigma}(x)
\,\mathrm{d}x
\right|
\,\mathrm{d}z.
\end{aligned}
\]
Using the triangle inequality and exchanging the order of integration,
\[
\begin{aligned}
&
\left\|
\mathcal{H}_sG_{\mu,\Sigma}
-
\mathcal{H}_s^\mu G_{\mu,\Sigma}
\right\|_{L^1}
\\
&\leq
\int_{\mathbb{R}^d}
G_{\mu,\Sigma}(x)
\int_{\mathbb{R}^d}
\left|
H_s(x,z-x)-H_s(\mu,z-x)
\right|
\,\mathrm{d}z\,\mathrm{d}x.
\end{aligned}
\]
Under the change of variables \(y=z-x\), this becomes
\[
\begin{aligned}
&
\left\|
\mathcal{H}_sG_{\mu,\Sigma}
-
\mathcal{H}_s^\mu G_{\mu,\Sigma}
\right\|_{L^1}
\\
&\leq
\int_{\mathbb{R}^d}
G_{\mu,\Sigma}(x)
\int_{\mathbb{R}^d}
\left|
H_s(x,y)-H_s(\mu,y)
\right|
\,\mathrm{d}y\,\mathrm{d}x.
\end{aligned}
\]
The Lipschitz assumption therefore gives
\[
\left\|
\mathcal{H}_sG_{\mu,\Sigma}
-
\mathcal{H}_s^\mu G_{\mu,\Sigma}
\right\|_{L^1}
\leq
L_s
\int_{\mathbb{R}^d}
\|x-\mu\|G_{\mu,\Sigma}(x)\,\mathrm{d}x.
\]
By the Cauchy--Schwarz inequality,
\[
\begin{aligned}
\int_{\mathbb{R}^d}
\|x-\mu\|G_{\mu,\Sigma}(x)\,\mathrm{d}x
&\leq
\left(
\int_{\mathbb{R}^d}
\|x-\mu\|^2G_{\mu,\Sigma}(x)\,\mathrm{d}x
\right)^{1/2}
\\
&\qquad\times
\left(
\int_{\mathbb{R}^d}
G_{\mu,\Sigma}(x)\,\mathrm{d}x
\right)^{1/2}.
\end{aligned}
\]
Since \(G_{\mu,\Sigma}\) has unit mass and
\[
\int_{\mathbb{R}^d}
\|x-\mu\|^2G_{\mu,\Sigma}(x)\,\mathrm{d}x
=
\operatorname{tr}\Sigma,
\]
we obtain
\[
\left\|
\mathcal{H}_sG_{\mu,\Sigma}
-
\mathcal{H}_s^\mu G_{\mu,\Sigma}
\right\|_{L^1}
\leq
L_s\sqrt{\operatorname{tr}\Sigma}.
\]
\end{proof}

%============================================================

\section{Experimental setup}

Across all experiments we test the ability of neural networks to recover the
deterministic and stochastic components of the Lévy-It\^o SDE given in \ref{eq:sde}.
Every experiment follows the
same pipeline: we generate a training and a test ensemble of the reference
process on a fine time grid (Euler--Maruyama integration with step
$\mathrm{d}t$, states saved every $\Delta t$); subsample in time; then apply a $z$-score normalization
$\tilde X = (X-\mu)/\sigma$. Unless noted otherwise we use
$N_{\text{train}}=1024$ and $N_{\text{test}}=256$ trajectories, integration step
$\mathrm{d}t=10^{-3}$, and each model is trained across $5$ random seeds.

\subsection{Lorenz system}
This benchmark probes chaotic, strongly nonlinear drift under multiplicative isotropic
noise. The state $X=(x,y,z)\in\mathbb{R}^3$ evolves as
\begin{equation}
  \dot x = \sigma_L(y-x),\qquad
  \dot y = x(\rho-z)-y,\qquad
  \dot z = xy-\beta z,
  \qquad g = \sigma \operatorname{diag}([x,y,z]),
\end{equation}
with the classical parameters $(\sigma_L,\rho,\beta)=(10,28,\tfrac{8}{3})$ and
noise scale $\sigma=3$. Trajectories are simulated on $t\in[0,2]$ with save
interval $\Delta t = 0.02$ from standard-normal initial conditions
$X_0\sim\mathcal{N}(0,\mathbb{I}_3)$.

\subsection{Multivariate Black--Scholes}
This benchmark exercises state-dependent, correlated (multiplicative) noise for
$n=3$ coupled assets $S\in\mathbb{R}^3_{>0}$:
\begin{equation}
  \mathrm{d}S_i = \mu_i S_i\,\mathrm{d}t
  + S_i\sum_{j} \big(\sigma_i L_{ij}\big)\,\mathrm{d}W_j,
  \qquad
  a_{ij}(S)=\sigma_i\,\rho_{ij}\,\sigma_j\,S_i S_j ,
\end{equation}
where $a=gg^\top$ is the state-dependent diffusion matrix and $L$ is the
Cholesky factor of the asset correlation matrix $\rho$. We use drift rates
$\mu=(0.05,0.06,0.04)$, volatilities $\sigma=(0.20,0.15,0.25)$, and correlations
$\rho_{12}=\rho_{23}=0.66$, $\rho_{13}=0.33$. Log-normal initial prices
$S_0 = \exp(0.1\,\mathcal{N}(0,\mathbb{I}_3))$ are integrated on $t\in[0,2]$ with
$\Delta t=0.02$; learning is performed in raw (unnormalized) coordinates. 

\subsection{Three-dimensional double-well jump-diffusion}

As a test problem combining metastability with discontinuous forcing, we consider a
three-dimensional gradient system whose first coordinate is bistable and whose remaining
coordinates are linearly confined. The dynamics follow the jump-diffusion stochastic
differential equation
\begin{equation}
  \mathrm{d}X_t \;=\; f(X_t)\,\mathrm{d}t \;+\; \sigma\,\mathrm{d}W_t \;+\; J\,\mathrm{d}N_t,
  \qquad
  f_1(x) = x_1\!\left(1 - x_1^2\right), \quad f_i(x) = -c\,x_i \;\; (i = 2,3),
\end{equation}
which is the gradient of the potential
$V(x) = \tfrac{1}{4}\left(x_1^2 - 1\right)^2 + \tfrac{c}{2}\sum_{i\ge 2} x_i^2$.
The drift $f$ has two stable fixed points at $(\pm 1, 0, 0)$ separated by an unstable
saddle at the origin, so the first coordinate $x_1$ constitutes the ``axis'' of the double
well while $x_2, x_3$ relax toward it with confinement strength $c$. The state is excited
continuously by isotropic Gaussian noise, with $W_t$ a three-dimensional Brownian motion of
per-component strength $\sigma$.

Superimposed on this continuous diffusion is a compound-Poisson jump term $J\,\mathrm{d}N_t$
that acts \emph{only along the double-well axis}: $N_t$ is a Poisson process of rate $\lambda$,
and at each event the first coordinate receives a symmetric Gaussian kick
$J \sim \mathcal{N}(0, \sigma_J^2)$, i.e. $x_1 \mapsto x_1 + J$, with the transverse
coordinates left unchanged. These kicks intermittently drive the particle over the potential
barrier, inducing well-to-well transitions that the pure-diffusion component would produce
only rarely, and endow the observed increments with a non-Gaussian, heavy-tailed character.
In our experiments we use $c = 1$, $\sigma = 0.5$, jump rate $\lambda = 1$, and jump scale
$\sigma_J = 1$ (comparable to the inter-well distance), integrating with the Euler--Maruyama
scheme at step $\mathrm{d}t = 10^{-3}$ over $t \in [0, 5]$ and recording snapshots every
$\Delta t = 0.05$.

We simulate $1024$ training and $256$ test trajectories from Gaussian-distributed initial
conditions and benchmark the neural-SDE architectures on their ability to recover both the
bistable drift and the jump-augmented noise structure from data alone. This setting is a
deliberately stringent test: a method that models the increments as purely Gaussian must
either underestimate the barrier-crossing frequency or inflate the diffusion coefficient,
whereas the jump events appear as anomalously large displacements localized to a single
coordinate. We report drift reconstruction error and (sliced) Wasserstein distance between
simulated and learned rollouts, averaged over independent training runs of each model.

\subsection{Coarse-graining (Molecular dynamics)}
Finally, we include a data-driven benchmark with no analytical ground-truth
drift, based on \cite{erban2014molecular}. A heavy tagged particle (mass $M$, radius $R$,
mass ratio $\mu=M/m$) collides elastically with an ideal gas of light point
particles; the bath is held in equilibrium by exact Poisson boundary influx and
collisions are resolved analytically within each step. 

It is known that as the particle density and relative mass go to infinity, the dynamics converge to the limiting Langevin process
\begin{equation}
  \mathrm{d}X_i = V_i\,\mathrm{d}t,\qquad
  \mathrm{d}V_i = -\gamma V_i\,\mathrm{d}t + \gamma\sqrt{2D}\,\mathrm{d}W_i ,\quad \text{as }  \mu\to\infty
\end{equation}
with friction $\gamma$ and diffusion $D=1$. However, there is analytical model for rarefied gases or for only moderately heavier tagged particles; in particular, dynamics in this rarefied regime are distinctly non-Brownian and involve jump-like phenomena.

We run an experiment in this rarefied non-Brownian regime with $\mu = 10$. We run 512 high-fidelity microscopic simulations of a three-dimensional particle bath with 20000 small particles in the bath and examine their effect on the tagged particle. 

The full 6D phase-space dynamics of the tagged particle is then exported so it can be fed to the models for training. Further details may be found in the repository.

\section{Implementation details}

The code for simulations and training was written in the Julia programming language. Julia is a high-level, high-performance language made for scientific computing. In particular, it has best-in-class stochastic simulation capabilities, which include accurate resolution of jumps, which are not available in mainstream ML languages such as Python. 

This has led to the baseline models (Euler-Maruyama, SDE-GAN and Trajecotry Flow Matching) being carefully ported and adapted to Julia. The ports may be found in the code repo. Any mistakes are our own.

\end{document}